\documentclass[11pt,technote,onecolumn]{IEEEtran}

\title{Frequency-Filtered Robust Tensor Principal Component Analysis}

\author{Shenghan Wang,
	Yipeng~Liu,~\IEEEmembership{Senior Member,~IEEE,}
	Lanlan~Feng,	
	~Ce~Zhu,~\IEEEmembership{Fellow,~IEEE}
	\thanks{This research is supported by National Natural Science Foundation
		of China (NSFC, No. 61602091), Sichuan Science and Technology Program (No. 2019YFH0008, No. 2018JY0035).}
	\thanks{All the authors are with School of Information and Communication Engineering,
		University of Electronic Science and Technology of China (UESTC), Chengdu, 611731, China,
		email: yipengliu@uestc.edu.cn}
}

\usepackage[utf8]{inputenc}
\usepackage[T1]{fontenc}
\usepackage{graphicx,subcaption,bm,mathtools,booktabs,tabularx,blindtext,multirow,dcolumn}
\usepackage{cite}
\usepackage{amsmath,amssymb,amsfonts,amsthm}
\usepackage{algorithm,algpseudocode}
\usepackage{setspace}
\usepackage{textcomp}
\usepackage{caption,xcolor}
\usepackage{enumitem}
\usepackage[colorlinks=true]{hyperref}
\setenumerate[1]{itemsep=0pt,partopsep=0pt,parsep=4pt,topsep=4pt}

\DeclareMathOperator*{\argmin}{argmin}
\DeclareMathOperator*{\minimize}{minimize}

\DeclareMathOperator{\diag}{diag}
\DeclareMathOperator{\sth}{sth}

\DeclareMathOperator{\FTNN}{FTNN}
\DeclareMathOperator{\ITNN}{ITNN}
\DeclareMathOperator{\TNN}{TNN}

\DeclareMathOperator{\FTSVT}{FTSVT}

\DeclareMathOperator*{\sgn}{sgn}

\DeclareMathOperator*{\fft}{fft}
\DeclareMathOperator*{\ifft}{ifft}

\DeclareMathOperator*{\conj}{conj}
\DeclareMathOperator*{\SVD}{SVD}
\DeclareMathOperator{\rank}{rank}
\newcommand{\abs}[1]{\left\lvert#1\right\rvert}
\newcommand{\norm}[1]{\lVert#1\rVert}
\newcommand{\tran}[1]{#1^{\mathrm{T}}}
\newcommand{\tranh}[1]{#1^{\mathrm{H}}}
\newcommand{\st}{{\text{s.t.}}}

\newsavebox\CBox
\newcommand{\e}[1]{\sbox\CBox{#1}\resizebox{\wd\CBox}{\ht\CBox}{\textbf{#1}}}

\def\a{\mathbf{a}}

\def\v{\mathbf{v}}
\def\vbar{\bar{v}}
\def\F{\mathbf{F}}
\newcommand{\Sm}{\mathbf{S}}
\newcommand{\Xm}{\mathbf{X}}
\newcommand{\Ym}{\mathbf{Y}}
\newcommand{\Lm}{\mathbf{L}}
\newcommand{\Am}{\mathbf{A}}
\newcommand{\Em}{\mathbf{E}}
\newcommand{\Um}{\mathbf{U}}
\newcommand{\Vm}{\mathbf{V}}

\def\Im{\mathbf{I}}
\def\A{\mathbf{\mathcal{A}}}
\def\B{\mathbf{\mathcal{B}}}
\def\C{\mathbf{\mathcal{C}}}

\def\E{\mathbf{\mathcal{E}}}

\def\X{\mathbf{\mathcal{X}}}
\def\U{\mathbf{\mathcal{U}}}
\def\V{\mathbf{\mathcal{V}}}

\def\R{\mathbf{\mathcal{R}}}
\def\Y{\mathbf{\mathcal{Y}}}

\def\L{\mathbf{\mathcal{L}}}
\def\S{\mathbf{\mathcal{S}}}
\newcommand{\Abar}{\mathbf{\mathcal{\bar{A}}}}
\newcommand{\Bbar}{\mathbf{\mathcal{\bar{B}}}}
\newcommand{\Cbar}{\mathbf{\mathcal{\bar{C}}}}

\newcommand{\Ubar}{\mathbf{\mathcal{\bar{U}}}}
\newcommand{\Vbar}{\mathbf{\mathcal{\bar{V}}}}
\newcommand{\Sbar}{\mathbf{\mathcal{\bar{S}}}}

\newcommand{\Xbar}{\mathbf{\mathcal{\bar{X}}}}
\newcommand{\Ybar}{\mathbf{\mathcal{\bar{Y}}}}

\newcommand{\Rbar}{\mathbf{\mathcal{\bar{R}}}}
\theoremstyle{plain} 
\newtheorem{defn}{Definition}
\newtheorem{lemma}{Lemma}
\newtheorem{theorem}{Theorem}

\begin{document}

\IEEEtitleabstractindextext{%
\begin{abstract}
Robust tensor principal component analysis (RTPCA) can separate the low-rank component and sparse component from multidimensional data. Its performance varies with different kinds of tensor decompositions, and the tensor singular value decomposition (t-SVD) is a popularly selected one due to its computational complexity. The standard t-SVD takes the discrete Fourier transform to exploit the residual in the 3rd mode, and all the frontal slices in frequency domain are optimized equally. In this paper, we incorporate the frequency filtering into t-SVD to enhance the RTPCA performance. Specially, different frequency bands are unequally treated with respect to corresponding physical meanings in the frequency-filtered tensor nuclear norm. The frequency-filtered tensor singular value thresholding can be deduced from the newly defined tensor nuclear norm optimization accordingly. The  obtained frequency-filtered RTPCA can be solved by alternating direction method of multipliers. It is the first time that frequency analysis is explicitly performed in tensor principal component analysis. With the prior knowledge between frequency bands, we propose the filtering strategy for different visual data processing tasks. Numerical experiments on synthetic 3D data, color image denoising and background modeling verify that the proposed method outperforms the state-of-the-art algorithms  in both accuracy and computational complexity.
\end{abstract}

}

\maketitle
\IEEEdisplaynontitleabstractindextext
\IEEEpeerreviewmaketitle

\section{Introduction}\label{sec:intro}

\IEEEPARstart{P}{rincipal} component analysis (PCA) \cite{PCA} is a classical dimension reduction technique that performs low-rank component extraction for a matrix. One of its main problems is the sensitivity to outliers. A number of improved PCA methods have been proposed to deal with it \cite{fischler1981random, huber1981robust, de2003framework, gnanadesikan1972robust, ke2005robust}. Among them, robust principal component analysis (RPCA) is the first polynomial-time algorithm with strong performance guarantees \cite{candes2011robust}, which has many successful applications including face recognition \cite{face, bao2012inductive, wang2010efficient},  background model initialization \cite{wright2009robust, li2004statistical, cao2015total} and image recovery \cite{gu2017weighted, zhao2014robust}. RPCA is designed for two-way data, but dimension reduction is required for many multi-way data, such as color images, color video sequences and hyperspectral images. The matrix computation based RPCA requires to reshape these multi-way data into matrices, which would lead to data structure information loss \cite{candes2011robust,lin2010augmented}.

Tensor is higher order generalization of vector and matrix \cite{Tensor_Decompositions, Tensor_Decompositions_1}, and it is a natural representation for multi-way data. In order to fully exploit the multidimensional structure of tensors, robust tensor principal component analysis (RTPCA) has been proposed by separating low-rank tensor component and sparse component of multi-way data \cite{lu2008mpca,lu2016tensor,cao2016total}. Fig.~\ref{pic:TRPCA} illustrates the model of RTPCA.

\begin{figure}[!t]
	\centering
	\includegraphics[width=.8\columnwidth]{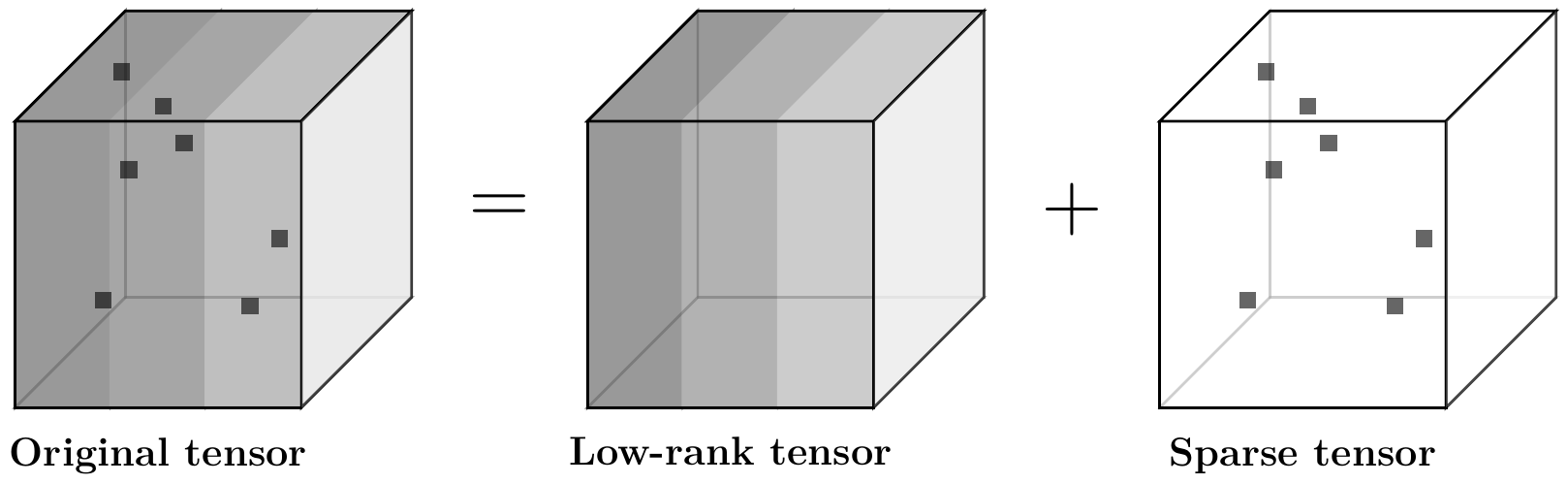}
	\caption{Illustration of RTPCA method}\label{pic:TRPCA}
\end{figure}

RTPCA methods vary with different tensor ranks. The canonical polyadic (CP) decomposition factorizes a tensor into the sum of several rank-one tensors \cite{lebedev2014speeding, luo2017tensor, liu2019low}. The CP rank of a tensor is defined as the smallest number of rank-one tensors. Methods based on CP decomposition often assume that the CP rank of the target tensor is known \cite{acar2011scalable,jain2014provable}. However, the CP rank of a tensor is NP-hard to compute \cite{de2008tensor}. Some works \cite{rai2014scalable,zhao2015bayesian} try to estimate the CP rank in Bayesian framework, but they often suffer from over-estimating or under-estimating the true CP rank. The Tucker decomposition decomposes a tensor into a core tensor multiplied by factor matrix along each mode \cite{delathauwer2000a, CPtuckerrank}. The Tucker rank is defined as the vector whose entries are the ranks of the factor matrices. As the nuclear norm is the convex envelope of the matrix rank, the sum of nuclear norms (SNN) has been proposed in the Tucker rank minimization problem \cite{liu2012tensor, goldfarb2014robust, huang2014provable}. In \cite{cao2016total,huang2014provable}, the authors successfully recover the low-rank and the sparse component from corrupted tensor based on SNN minimization. However, the Tucker decomposition has to unfold the tensors into matrices which needs high computation costs.

Recently the tensor singular value decomposition (t-SVD)  has been proposed \cite{kilmer2011factorization, braman2010third, kilmer2013third}, which factorizes a order-3 tensor as the tensor products of three tensors, as it shows in Fig.~\ref{pic:t-SVD}. The tubal rank, which is defined in t-SVD framework, can characterise the low-rank structure of a tensor very well \cite{zhang2014novel}. The tubal rank minimization problem is usually relaxed into the tensor nuclear norm (TNN) minimization problem, which can be solved by convex optimization. There are different forms of TNN in the t-SVD based RTPCA \cite{lu2019tensor, chen2017iterative, feng2020robust, liu2018improved, kong2018t,hu2016moving}, and several sparse constraints are used according to different applications \cite{zhang2014novel, zhou2017outlier, baburaj2016reweighted}. Though t-SVD can be calculated easily in the Fourier domain, none of these works exploit the prior knowledge about frequency spectrum even when it is available in some applications.

Since larger singular values of a matrix indicate its underlying principal directions, it is unfair to shrink all the singular values with a same threshold when solving the nuclear norm minimization problem. Therefore, weighted nuclear norm (WNN) is proposed to adaptively assign weights on different singular values \cite{gu2017weighted}. This data-driven strategy selects weights in each iteration according to the size of singular values in last iteration. 

Some similar weighting schemes have been successfully established for tensor analysis in recent works \cite{xu2019laplace,madathil2018twist,madathil2018dct}. In \cite{madathil2018twist}, the authors extend the WNN to the tensor case as weighted tensor nuclear norm (WTNN). Each singular value is assigned a corresponding weight, which depends on the value of the singular value itself. In \cite{xu2019laplace}, the authors find that the Laplace function, defined as $\phi(x)=1-\bm{e}^{-x/\epsilon}$, is a better approximation to the $\ell_0$-norm. So that the sum of the Laplace function of singular values is a better surrogate for tensor multi-rank than TNN. Thus they propose the $\epsilon-\TNN$ based on the Laplace function. In \cite{madathil2018dct}, the authors propose a weighted adaptive $m$-mode cosine transform tensor SVD methods. And the $m$-mode cosine transform tensor nuclear norm ($m-\TNN$) is defined as the sum of tensor nuclear norms at each mode. However, these methods do not explicitly take the physical meaning of the data itself into consideration, especially in the perspective of frequency filtering,  leading to insufficiently exploiting the prior information to process high-dimensional data.

We find that there exists prior information between the frequency bands of many visual data. However, existing tensor based methods treat all the bands equally and neglect these prior knowledge. Motivated by the extensive applications of Fourier analysis, we introduce the frequency filtering into the t-SVD framework. To our best of knowledge, this is the first work to investigate the RTPCA problem via the viewpoint of frequency filtering in t-SVD framework.

\subsection{Contributions}

This work mainly has several contributions.

Firstly, we broaden the t-SVD framework in analysis part. We rigorously analyze the essence of t-SVD and define frequency-filtered tensor nuclear norm (FTNN) by incorporating the low-rank approximation and the frequency filtering in a unified formula. The proposed FTNN can be explained from the perspective of filtering and it makes full use of the prior knowledge in the frequency domain. It is the first time that frequency analysis is taken in tensor principal component analysis

Secondly, we rigorously deduce the proximal operator of FTNN, named as FTSVT, to implement the Fourier filtering. Then we incorporate the proposed FTNN into a robust tensor principal component analysis (RTPCA) model, denoted as FTNN-RTPCA. And the FTSVT operator is applied to solve the FTNN minimization problem.

Thirdly, we propose the filtering strategy for FTNN and verify it in synthetic 3D data experiment. In practice, we explicitly point out what the prior information is and discuss about how to acquire prior knowledge. 

Finally, we employ the proposed FTNN-RTPCA model to deal with some visual tasks. 
\begin{itemize}
	\item In color image denoising experiment, the prior knowledge is obtained by statistical analysis of several images. Then we design a filtering scheme to deal with images corrupted by different ratio of sparse noise. Results show the superior performance of FTNN-RTPCA compared with existing methods.
	\item In background modeling experiment, the prior knowledge is obtained by experimental analysis. Later we design a special filtering coefficients and develop a SVD-free algorithm. Results show the performance by FTNN-RTPCA is better than some state-of-art methods, especially in running time and computational complexity.
\end{itemize} 

\subsection{Organization}

The rest of this paper is organized as follows. Section \ref{sec:notations} gives some notations and preliminaries used in this paper. In Section \ref{sec:model}, we define a new FTNN, and propose the FTNN-RTPCA method. In Section \ref{sec:experiments}, we first discuss the setting of the filtering vector and verify it in synthetic experiments. Then we conduct some numerical experiments to compare the proposed method with some state-of-the-art algorithms including color image denoising and background model initialization. Section \ref{sec:conlusion} gives a conclusion of this paper.

\section{Preliminaries and Related Works}\label{sec:notations}


\subsection{Notation}

A scalar, a vector, a matrix and a tensor are denoted as $a$, $\a$, $\Am$, $\A$. $\A^{(i_3)}$ denotes the $i_3$-th frontal slice of $\A$. The tube fiber on the third mode is denoted as $\A(i_1,i_2,:)$.

The $\ell_1$ norm and the Frobenius norm of a tensor are defined as $\norm{\A}_1=\sum_{i_1,i_2,i_3}\abs{\A_{i_1,i_2,i_3}}$ and $\norm{\A}_\text{F}=\sum_{i_1,i_2,i_3}\abs{\A_{i_1,i_2,i_3}}^2$, respectively. When $I_3=1$, these tensor norms degenerate into corresponding matrix norms. The nuclear norm of a matrix is $\norm{\Am}_* = \sum_{i_0} \sigma_{i_0} (\Am)$, where $\sigma_{i_0} (\Am) $ with $ I_0 = {\operatorname{min}\{I_1, I_2 \} } $, is the $i_0$-th singular values. 

For a matrix $\Am$, $\tranh{\Am}$ takes the complex conjugate transpose of all its entries, and $\tran{\Am}$ is the transpose matrix. $\conj(\Am)$ takes the complex conjugate matrix of $\Am$. The sets of real and complex numbers are $\mathbb{R}$ and $\mathbb{C}$, respectively.  $\left\lceil x\right\rceil$ is the nearest integer which is equal or greater than $x$.

\subsection{Discrete Fourier Transformation}\label{sec:DFT}
As the discrete Fourier transformation (DFT) plays a vital role in the t-SVD, we provide a brief introduction of the notations and definitions about it. The frequency bands and frequency components of a order-3 tensor are defined in this subsection.

The discrete Fourier transformation (DFT) matrix is defined as:
\begin{align*}
	 \F_N = \begin{bmatrix}
	 1 & 1 & 1 & \cdots & 1 \\
	 1 & \omega & \omega^2 & \cdots & \omega^{N-1} \\
	 \vdots & \vdots & \vdots & \ddots & \vdots \\
	 1 & \omega^{N-1} & \omega^{2(N-1)} & \cdots &\omega^{(N-1)(N-1)}
	 \end{bmatrix}\in\mathbb{C}^{N\times N}
\end{align*}
where $\omega=\operatorname{exp}({-\frac{2\pi j}{N}})$, where $ j = \sqrt{-1} $. We can find that:
\begin{equation}
	\F_N\tranh{\F_N} = \tranh{\F_N}\F_N = N\Im, ~\F_N^{-1} = \frac{1}{N}\tranh{\F_N}
\end{equation}

The DFT for a vector $\v=\tran{[v_1, v_2, \ldots , v_N]}$ is denoted as $\bar{\mathbf{\v}} = \F_N\v\in\mathbb{C}^N$, and the inverse DFT is represented as $\v = \F_N^{-1}\bar{\mathbf{\v}}\in\mathbb{R}^N$. These two transforms can be efficiently calculated  by fast Fourier transformation (FFT) and inverse fast Fourier transformation (IFFT). 

As it can be deduced from the definition of DFT, we have  one property as follows \cite{rojo2004some}:
\begin{equation}
\vbar_1\in\mathbb{R};~ \\
\conj(\vbar_i) = \vbar_{N-i+2}, n=2,3,\cdots,\left\lceil\frac{N+1}{2}\right\rceil \label{eq:symm}	
\end{equation}
where $ \bar{\v}=\tran{[\bar{v}_1, \bar{v}_2, \ldots , \bar{v}_N]} $.

To further analyze the spectral properties of this vector, we give the definitions of frequency band and frequency component.

\begin{defn}[\emph{\e{frequency band}}]\label{def:freqband}
	For a vector $\v\in\mathbb{R}^N$, based on the property (\ref{eq:symm}), the frequency band is defined as
	\begin{equation}
		\left\{\mathbf{\vbar}\right\}_{i} = \begin{cases}
			 \vbar_i  &  i=1~\text{or}~i = N-i+2,\\
			(\vbar_i,\vbar_{N-i+2}) & else.
		\end{cases}\label{eq:freqband}
	\end{equation}
\end{defn}
When $N$ is odd, the corresponding frequency bands are $(\vbar_1)$, $(\vbar_2,\vbar_N)$, $\cdots$, $(\vbar_{\frac{N+1}{2}},\vbar_{\frac{N+3}{2}})$;  when $N$ is even, they are $(\vbar_1)$,$(\vbar_2,\vbar_N)$,$\cdots$,$(\vbar_{\frac{N+2}{2}})$. There are always $N_v = \lceil\frac{N+1}{2}\rceil$ frequency bands in all.

\begin{defn}[\emph{\e{frequency component}}]
	For a vector $\v\in\mathbb{R}^N$, a frequency band	$\left\{\mathbf{\vbar}\right\}_{i}$ in the Fourier domain can be transformed into a frequency component $[\v]_{i}\in\mathbb{R}^N$ in the time domain as follows
	\begin{equation*}
	\begin{split}
	[\v]_{i} &= \ifft(\left\{\mathbf{\vbar}\right\}_{i}) \\
			   &= \frac{1}{N}\tranh{\F_N}\tran{[0,\cdots,0,\vbar_i,0,\cdots,0,\vbar_{N-i+2},0,\cdots]} 	
	\end{split}
	\end{equation*} 
\end{defn}

Among all the components, the zero-frequency component $[\v]_1$ is 
\begin{equation}
	[\v]_1 = \tran{[\frac{1}{N}\sum_{i=1}^N v_i,\frac{1}{N}\sum_{i=1}^N v_i,\cdots,\frac{1}{N}\sum_{i=1}^N v_i]}
 	\label{eq:zero}
\end{equation}

It is distinctive because it assembles all the energy and indicates the average energy of a vector. For example, a vector $\mathbf{x} = \tran{[4,6,4,6]}$ has 3 frequency components $[\mathbf{x}]_1 = \tran{[5,5,5,5]}$, $[\mathbf{x}]_2 = \tran{[0,0,0,0]}$, and $[\mathbf{x}]_3 = \tran{[-1,1,-1,1]}$ respectively. Intuitively, zero frequency component can be regarded as the average energy.

\begin{lemma}\label{lemma:1}
	A vector $\v\in\mathbb{R}^N$ can be represented as the sum of all frequency components. 
	\begin{equation}
	\begin{split}
		\v &= \ifft(\left\{\mathbf{\vbar}\right\}_1) + \ifft(\left\{\mathbf{\vbar}\right\}_2) + \cdots + \ifft(\left\{\mathbf{\vbar}\right\}_{N_v}) \\
		   &= [\v]_1 + [\v]_2 + \cdots + [\v]_{N_v} \label{eq:fft1}
	\end{split}
	\end{equation}
\end{lemma}

Consider a tensor $\A\in\mathbb{R}^{I_1\times I_2\times I_3}$, and $\Abar$ as the result of the DFT along the third dimension of $\A$. The DFT and inverse DFT on a tensor can be represented as follows:
\begin{equation*}
	\Abar = \fft(\A,[],3), \A = \ifft(\Abar,[],3)	
\end{equation*}
For convenience, we  abbreviate it as $\Abar = \fft(\A), \A = \ifft(\Abar)$ in this paper.

The concepts of frequency band and frequency component can be generalized to tensor data. Fig.~\ref{pic:fft} shows an illustration of the DFT on a tensor $\A\in\mathbb{R}^{I_1\times I_2\times I_3}$. $I_1 I_2$ DFTs need to be performed along the third mode, and $I = \lceil\frac{I_3+1}{2}\rceil$ frequency bands can be obtained. 

\begin{figure}[!t]
	\centering
	\includegraphics[width=.6\linewidth]{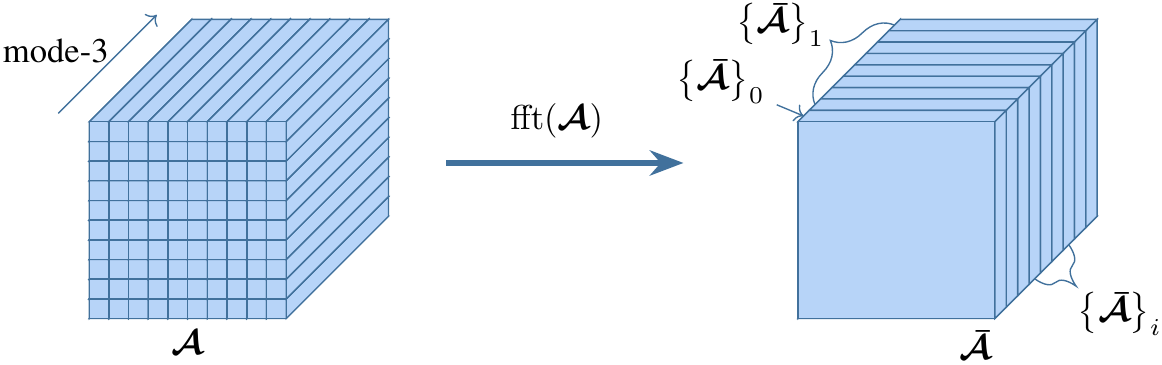}\hspace{1pt}\\
	\caption{Illustration about the DFT on a tensor.}\label{pic:fft}
\end{figure}

\subsection{Preliminaries about t-SVD}

\begin{defn}[\emph{\e{t-product}}\cite{kilmer2011factorization}]
	Given $\A\in\mathbb{R}^{I_1\times I_2\times I_3}$ and $\B\in\mathbb{R}^{I_2\times I_4\times I_3}$, the t-product between $\A$ and $\B$ is defined as follows:
	$$\C(i,j,:) = \sum_{k=1}^{I_2}\A(i,k,:)\bullet\B(k,j,:)$$
	where $\bullet$ denotes the circular convolution between two tubes on the thrid mode and $\C\in\mathbb{R}^{I_1\times I_4\times I_3}$. When converted into Fourier domain, the t-product can be calculated by matrix product on each frontal slice separately as $\Cbar^{(i_3)} = \Abar^{(i_3)}\Bbar^{(i_3)}$.
\end{defn}

\begin{defn}[\emph{\e{t-SVD}}\cite{kilmer2011factorization}]
	For a tensor $\A\in\mathbb{R}^{I_1\times I_2 \times I_3}$, the t-SVD of $\A$ is :
	$$\A = \U*\S*\tranh{\V}$$
	where $\U$ and $\V$ are orthogonal tensors of size $ I_1\times I_1\times I_3$ and $ I_2\times I_2 \times I_3$ respectively. $\S \in \mathbb{R}^{I_1 \times I_2 \times I_3}$ is an F-diagonal tensor.
\end{defn}

Fig.~\ref{pic:t-SVD} shows an illustration of the t-SVD. It can be easily calculated in the Fourier domain. We can get the frontal slices of $\Ubar,\Sbar,\Vbar$ by 
$$\SVD(\Abar^{(i_3)}) = [\Ubar^{(i_3)},\Sbar^{(i_3)},\Vbar^{(i_3)}]$$

\begin{defn}[\emph{\e{tensor nuclear norm: TNN}}\cite{lu2019tensor}]
	The tensor nuclear norm of a tensor $\A$ based on t-SVD is defined as
	\begin{equation}
	\norm{\A}_{\TNN} = \frac{1}{I_3}\sum_{i_3=1}^{I_3} \norm{\Abar^{(i_3)}}_*  \label{eq:old_TNN}
	\end{equation}
\end{defn}

\subsection{Related Works}

The basic assumption of RPCA problem is that the data matrix can be decomposed into a low-rank matrix $\Lm$ and a sparse matrix $\Em$, i.e., $\Xm = \Lm+\Em$. Candes {\itshape et al}.\cite{candes2011robust} have proven that under some incoherent conditions, $\Lm$ and $\Em$ can be recovered by solving the convex optimization problem as follows:
\begin{equation}
\minimize_{\Lm,~\Em}\ \norm{\Lm}_* + \lambda\norm{\Em}_1\ \st\ \Xm = \Lm + \Em,
\end{equation}
where $\norm{\cdot}_*$ is the nuclear norm which represents the sum of singular values of low-rank matrix $\Lm$,~$\norm{\cdot}_1$ is $\ell_1$ norm which represents the sum of absolute values of all entries in sparse matrix $\Sm$. The parameter $\lambda$ is set to balance the two terms. 

Robust tensor principal component analysis (RTPCA) assumes a tensor $\X\in\mathbb{R}^{I_1 \times I_2 \times I_3}$ can be formulated as 
\begin{equation}
\X = \L + \E,
\label{eq: TRPCA_all}
\end{equation} 
where $\L$ is a low-rank tensor and $\E$ is a sparse tensor.  Several different  nuclear norms for tensor have been proposed to solve the problem (\ref{eq: TRPCA_all}). Below are the details of these methods.

\begin{figure}[!t]
	\centering
	\includegraphics[width=.8\columnwidth]{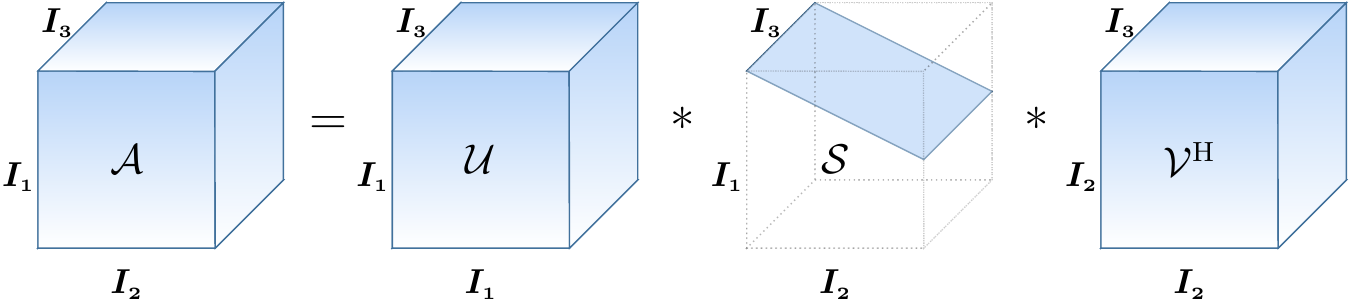}
	\caption{Illustration of t-SVD of a tensor $\A\in\mathbb{R}^{I_1\times I_2\times I_3}$}\label{pic:t-SVD}
\end{figure}

The Tucker decomposition based RTPCA can be formulated as follows \cite{liu2012tensor}:
\begin{equation}
	\minimize_{\L,~\E}\ \sum_{d=1}^{D} \lambda_d\norm{\L^{\{d\}}}_* + \norm{\E}_1,\ \st\ \X = \L + \E
\end{equation}
where $\L^{\{d\}}$ is the mode-$d$ matricization of the low-rank tensor $\L$. For example, when given a order-3 tensor, it uses the combination of three matrix nuclear norms to solve the Tucker rank minimization problem. 

Based on t-SVD, a new tensor nuclear norm (TNN) is rigorously deduced, and the corresponding convex optimization model (RTPCA-TNN) is given as follows \cite{lu2019tensor}:
\begin{equation}
	\minimize_{\L,\E}\ \norm{\L}_{\TNN} + \lambda\norm{\E}_1,\ \st\ \X = \L + \E.  
\end{equation}
where $\lambda$ is a regularization parameter, and  suggested to be set as $1/{\sqrt{\max(I_1,I_2)I_3}}$ to guarantee the exact recovery. $\norm{\L}_{\TNN}$ is the TNN for the low-rank tensor $\L$. It has been proven that the low -rank tensor and the sparse tensor can be perfectly recovered under some certain tensor incoherence conditions. The RTPCA-TNN method has been applied to solve the image recovery and background modeling problem successfully.

To further exploit the low-rank structures in multi-way data, the improved tensor nuclear norm (ITNN) is proposed, and defined as the weighted sum of TNN and nuclear norm of core matrix as follows \cite{liu2018improved}:
\begin{equation}
	\norm{\X}_{\text{ITNN}} = \norm{\X}_{\text{TNN}} + \lambda_S\norm{\bar{\Sm}}_*
\end{equation}
where $\lambda_S$ is set to balance the two terms. Fig.~\ref{pic:ITRPCA} shows how to obtain the core matrix from the core tensor.  Accordingly, a new ITNN-RTPCA method can be obtained: 
\begin{equation}
	\minimize_{\L,~\E}\ \norm{\L}_{\ITNN} + \lambda\norm{\E}_1,\ \st\ \X = \L + \E. 
\end{equation}

\begin{figure}[!t]
	\centering
	\includegraphics[width=.4\columnwidth]{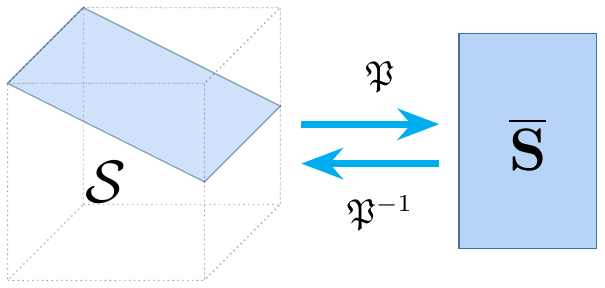}
	\caption{Illustration of transforms between core tensor $\S\in\mathbb{R}^{I_1\times I_2\times I_3}$ and core matrix $\bar{\Sm}\in\mathbb{R}^{\min\{I_1,I_2\}\times I_3}$} \label{pic:ITRPCA}
\end{figure}

\section{Frequency-Filtered Tensor Nuclear Norm (FTNN)}\label{sec:model}

\subsection{Framework of FTNN}

For a tensor $\X\in\mathbb{R}^{I_1\times I_2\times I_3}$, $\Xbar$ is the result of the DFT on the third dimension of $\X$. According to the definition \ref{def:freqband}, a frequency band of $\X$ is defiend as a conjugate pair $\{\Xbar\}_{j} = \left(\Xbar^{(j)},\Xbar^{(I_3-j+2)}\right)$ and when $j=1$ or $j=I_3-j+2$, $\{\Xbar\}_{j} = \Xbar^{(j)}$.

Then the frequency-filtered tensor nuclear norm (FTNN) is defined as:
\begin{equation}
\norm{\X}_{\FTNN} =\frac{1}{I}\sum_{j=1}^{ I}\alpha_j\norm{\{\Xbar\}_j}_*,\quad I=\left\lceil\frac{I_3+1}{2}\right\rceil\label{eq:ftnn1}
\end{equation}
where $\norm{\{\Xbar\}_j}_* = \norm{\Xbar^{(j)}}_* + \norm{\Xbar^{(I_3-j+2)}}_*$ is the sum of two matrix nuclear norms, $I$ denotes the number of frequency bands, and $\alpha_j$ is a pre-defined parameter assigned to the $j$-th frequency band, whose value depends on the prior knowledge. Such a definition emphasizes the significance of frequency bands and brings a new perspective to analyze low-rank tensor approximation problems.

Fig.~\ref{fig:frameworkFTNN} shows the framework of proposed FTNN. In t-SVD framework, the calculation of t-SVD consists of the Fourier transform and the matrix SVD, which are widespread techniques in signal processing. Most of previous works improve the part on SVD to make more use of low-rank prior information of data. However, the part on Fourier transform is less addressed, and the prior information between frequency bands is often ignored. We define the FTNN by incorporating the low-rank approximation and the frequency filtering in a unified formula. In (\ref{eq:ftnn1}), the factors $\alpha_j, 1 \leq j \leq I$ corresponds to frequency filtering, and $\norm{\{\Xbar\}_j}_*, 1 \leq j \leq I$, corresponds to low-rank matrix approximation. With the proposed FTNN, the prior information is utilized by means of frequency filtering.

\begin{figure}[!t]
	\centering
	\includegraphics[width=.5\linewidth]{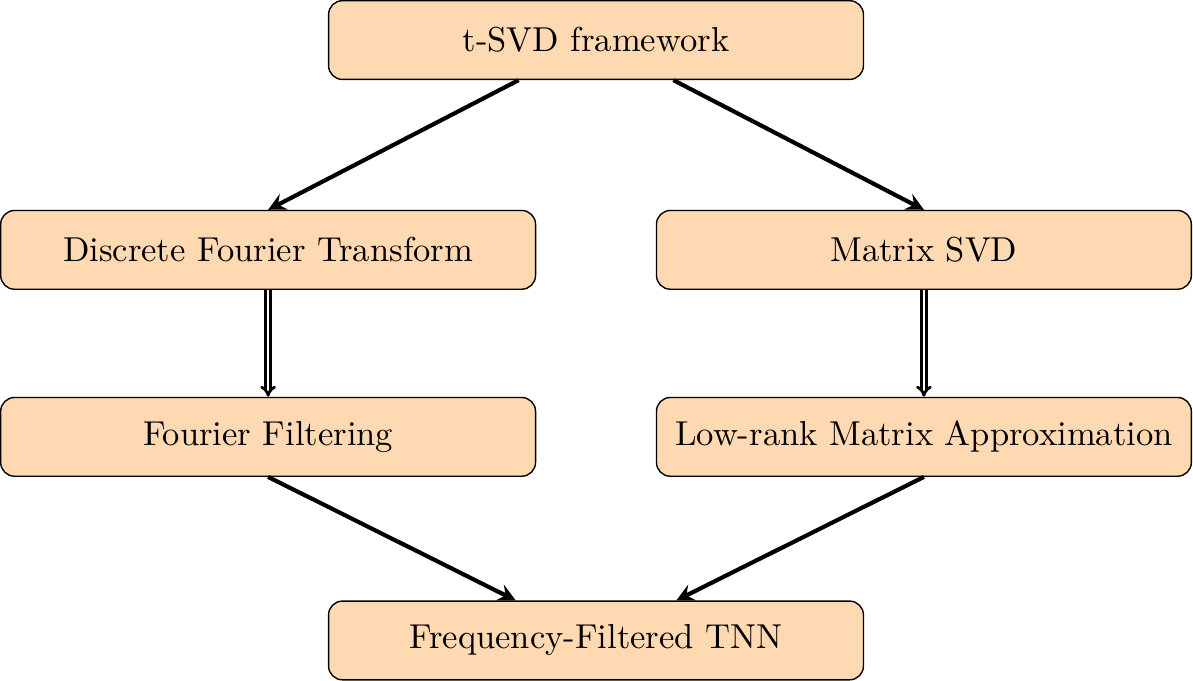}
	\caption{Framework of FTNN}\label{fig:frameworkFTNN}
\end{figure}

\subsection{Two Viewpoints of FTNN}

\subsubsection{Low-rank Tensor Approximation}

FTNN can simply be considered as a frequency band weighted tensor nuclear norm, whose purpose is to solve the low-rank tensor approximation problems. And the nuclear norm $\norm{\Xbar^{(i)}}_*$ is the convex relaxation to $\rank(\Xbar^{(j)})$. Since different frequency bands have different physical meanings and are affected differently by noise, the filtering vector $\bm{\alpha}=\tran{[\alpha_1,\alpha_2,\dots,\alpha_I]}$ is set to balance the nuclear norms of different slices.

In this case, the minimization of FTNN is a convex relaxation of weighted tensor rank minimization.

\subsubsection{Fourier Filtering}

FTNN can also be considered as a frequency filtering method. In the definition of FTNN (\ref{eq:ftnn1}), $\X$ is a signal represented in tensor format; the matrix nuclear norm $\norm{\{\Xbar\}_j}_*$ characterizes the amplitude or energy of the signal in the $j$-th band; $\alpha_j$ characterizes the filtering coefficient for the $j$-th frequency component.

The minimization of FTNN can be regarded as a filltering process. In this process, the vector $\bm{\alpha}=\tran{[\alpha_1,\alpha_2,\ldots,\alpha_{ I}]}$ controls the relative filtering thresholds for all the frequency bands; the minimization $\norm{\{\Xbar\}_j}_*,1\le j\le I$ can be regarded as a filtering process for the corresponding frequency band. The technical details for filtering can refer to Section~\ref{sec:ftsvt}. In this way, FTNN incorporates the low-rank approximation and Fourier filtering together. Equipped with the proposed FTNN, t-SVD becomes more flexible and powerful to deal with various data.

Our method is essentially a weighted combination of the nuclear norms in the transform domain, whose weights are chosen by the physical meanings according to the transforms. Any orthogonal transforms can be applied to the proposed method, such as discrete cosine trasformation (DCT) and discrete wavelet transform (DWT). The selection of DFT has two advantages. On the one hand, it possesses a fast operation mechanism by FFT, which can bring computational efficiency; On the other hand, it has the theoretical support provided by Fourier analysis theory, by which the prior knowledge in the frequency domain can be well explored.

\subsection{Frequency-Filtered Tensor Singular Value Thresholding}\label{sec:ftsvt}

Similar to tensor singular value thresholding (TSVT) operator related to TNN \cite{lu2019tensor}, the frequency-filtered tensor singular value thresholding (FTSVT) operator is strictly derived. The proximal operator of FTNN is computed as follows:
\begin{equation}
\min_{\X}\ \tau\norm{\X}_{\FTNN} + \frac{1}{2}\norm{\X-\Y}_F^2 \label{eq: ftsvtsolve}
\end{equation}

Denote $\Y = \U*\S*\tranh{\V}$ as the t-SVD of tensor $\Y$. For a threshold $\tau>0$, the FTSVT operator is defined as follows:
\begin{equation}
\FTSVT_{\tau\bm{\alpha}}(\Y) = \U*\S_{\star}*\tranh{\V} 
\end{equation}
where
\begin{equation}
\S_{\star} = \ifft(\Sbar_{\star},[],3)
\end{equation}
And each frontal slice of $\Sbar_{\star}^{(i)}$ satisfy $\Sbar_{\star}^{(i)} = (\Sbar^{(i)} - \tau\alpha_i)_{+}$. Note that $x_+$ takes the positive part of $x$, {\itshape i.e.}, $x_+ = \max(x,0)$. The FTSVT operator is the proximal operator associated with FTNN.

\begin{theorem}
	For any positive threshold $\tau>0$ and $\Y\in\mathbb{R}^{I_1\times I_2\times I_3}$, the frequency-filtered tensor singular value thresholding operator obeys
	\begin{equation}
	\FTSVT_{\tau\bm{\alpha}}(\Y) = \argmin_{\X}\ \tau\norm{\X}_{\FTNN} + \frac{1}{2}\norm{\X-\Y}_F^2\label{eq:ftsvt1}
	\end{equation}
\end{theorem}
\begin{proof}
	Denote $\Y = \U*\S*\tranh{\V}$ as the t-SVD of $\Y$. Since each slice $\Sbar^{(i)}$ is get from $[\Ubar^{(i)},\Sbar^{(i)},\Vbar^{(i)}] = \SVD(\Ybar)$, thus $\Sbar$ and $\Sbar_{\star}$ are real. By property (\ref{eq:symm}), $\FTSVT_{\tau\bm{\alpha}}(\Y)$ is real. Finally by properties (34) and (12) in \cite{lu2019tensor}, problem (\ref{eq:ftsvt1}) is equivalent to
	\begin{align}
	&\quad \argmin_{\X}\ \tau\norm{\X}_{\FTNN} + \frac{1}{2I_3}\norm{\Xbar - \Ybar}_\text{F}^2 \notag\\
	&\Leftrightarrow \argmin_{\X}\ \frac{1}{I}\sum_{j}^{I}\tau\alpha_j\norm{\{\Xbar\}_j}_* + \frac{1}{2I_3}\norm{\Xbar- \Ybar}_\text{F}^2 \notag\\
	&\Leftrightarrow \argmin_{\X}\ \frac{1}{I_3}\sum_{i_3=1}^{I_3}\left(\tau\alpha_{i_3}\norm{\Xbar^{(i_3)}}_* + \frac{1}{2}\norm{\Xbar^{(i_3)}-\Ybar^{(i_3)}}_\text{F}^2\right) \label{eq:FTSVT}
	\end{align}
	Therefore, the problem in (\ref{eq:ftsvt1}) can be divided into $I_3$ subproblems. By Theorem 2.1 in \cite{cai2010singular}, we can see that the $i$-th frontal slice of $\fft\left(\FTSVT_{\tau\bm{\alpha}}(\Y),[],3\right)$ solve the $i$-th subprolem of (\ref{eq:FTSVT}). Hence $\FTSVT_{\tau\bm{\alpha}}(\Y)$ solve the problem (\ref{eq:ftsvt1}).
\end{proof}

The computational details about FTSVT operator are given in Algorithm \ref{algor:FTNN}. Notice that when $\Rbar^{(i)}=0$,  the $i$-th frequency band $\{\Rbar\}_{i-1}$ is totally preserved, and this frequency band is totally discarded for $\Rbar^{(i)} = +\infty$. We do not need to compute SVD at this point.
\begin{algorithm}[!t]
	\setstretch{1.1}
	\caption{FTSVT operator}\label{algor:FTNN}
	\begin{algorithmic}
		\State\textbf{Input:} Tensor $\X\in\mathbb{R}^{I_1\times I_2\times I_3}$, Filtering vector $\bm{\alpha}$, Threshold $\tau$
		\State Compute $\Xbar = \fft(\X)$
		\For {$i = 1,\ldots,\left\lceil\frac{I_3+1}{2}\right\rceil$}
		\If{$\alpha_i=0$}
		\State $\Rbar^{(i)} = \Xbar^{(i)}$
		\ElsIf{$\alpha_i = \infty$}
		\State $\Rbar^{(i)} = 0$
		\Else 
		\State $[\Um, \Sm, \Vm] = \SVD(\Xbar^{(i)})$
		\State $\Rbar^{(i)} = \Um\cdot(\Sm-\alpha_i\tau)_+\cdot\tran{\Vm}$
		\EndIf
		\EndFor
		\For {$i = \left\lceil\frac{I_3+1}{2}\right\rceil+1,\ldots,I_3$}
		\State $\Rbar^{(i)} = \conj(\Rbar^{(I_3- i +2)})$
		\EndFor
		\State $\R = \ifft(\Rbar)$		
		\State\textbf{Output:} $\FTSVT_{\tau\bm{\alpha}}(\Y) = \R$
	\end{algorithmic}
\end{algorithm}

The proximal operator FTSVT implements the frequency filtering. Next we will discuss the components that are filtered during each filtering iteration. Suppose a matrix  $\Ym=\Um\Sm\tran{\Vm}$, where $\Sm=\diag(\sigma_1,\sigma_2,\ldots,\sigma_r)$. We select a truncation threshold $\tau$, which satisfies $\sigma_1\ge\sigma_2\ge\cdots\ge\sigma_i\ge\tau\ge\sigma_{i+1}\ge\cdots\ge\sigma_r$. So that the singular value threshold (SVT) \cite{cai2010singular} operator can be regarded as a filtering process as:
\begin{align}
	\Xm_{\star} &= \argmin_{\Xm}~\tau\norm{\Xm}_* + \frac{1}{2}\norm{\Xm-\Ym}_F^2 \nonumber\\
	&= \Um\diag(\sigma_1-\tau,\ldots,\sigma_i-\tau,0,\ldots,0)\tran{\Vm}\nonumber\\
	&= \Ym - \Um\diag(\tau,\ldots,\tau,\sigma_{i+1},\ldots,\sigma_r)\tran{\Vm}\nonumber\\
	&= \Ym - \Ym_{\tau}
\end{align}
where $\Ym_{\tau}$ represents the component that is filtered out of $\Ym$. It corresponds to the filtering threshold $\tau$.

Suppose there are $I$ frequency bands in all. By selecting a series of different filtering coefficients $\alpha_j,1\leq j\leq I$, the filtering thresholds of FTNN for different frequency bands are as follows:
\begin{equation*}
\bm{\tau}_{\FTNN} = [\alpha_1\cdot\tau,~\alpha_2\cdot\tau,~\ldots,~\alpha_{I}\cdot\tau].
\end{equation*}

Here all the filtering coefficients are pre-defined. For the $j$-th band with threshold $\alpha_j\cdot\tau$, the result of filtering is $\Xbar_{\star}^{(j)} = \Ybar^{(j)} - \Ybar_{\alpha_j\tau}^{(j)}$. By adjusting the filtering coefficients, the filtering for different frequency bands can be realized. With the iterations going on, $\tau$ is set to decrease gradually so that the components which are filtered will tend to be smaller.

\subsection{Frequency Components Analysis (FCA)}\label{sec:FCA}

According to the analysis results about the Fourier transform in Section \ref{sec:DFT}, when we take the DFT along the third mode of a order-3 tensor, many frequency bands in the Fourier domain can be obtained and each frequency band  corresponds to a frequency component in the time domain. In other words, given a tensor $ \X \in \mathbb{R}^{I_1\times I_2\times I_3}$, it can be decomposed into $I$ frequency components via Fourier analysis:
\begin{equation*}
	\X = [\X]_1 + [\X]_2 + \cdots + [\X]_{I} \eqno{(12)},
\end{equation*}
where $I = \left\lceil\frac{I_3+1}{2}\right\rceil$.  

\begin{figure}[!t]
	\centering
	\subcaptionbox{$\sum_{i=1}^{46}[\X]_i$}{\includegraphics[width=.26\linewidth]{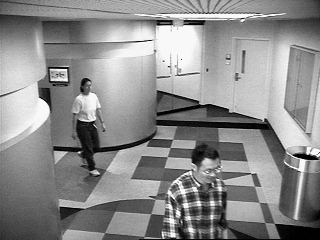}}
	\subcaptionbox{$[\X]_1$}{\includegraphics[width=.26\linewidth]{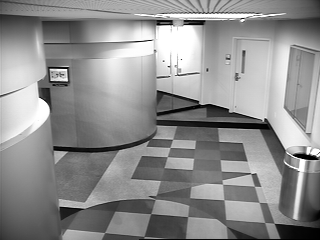}}
	\subcaptionbox{$\sum_{i=2}^{46}[\X]_i$}{\includegraphics[width=.26\linewidth]{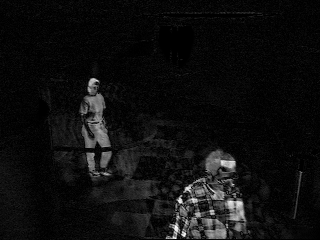}}\\[5pt]
	\subcaptionbox{$[\X]_1 + [\X]_2$}{\includegraphics[width=.26\linewidth]{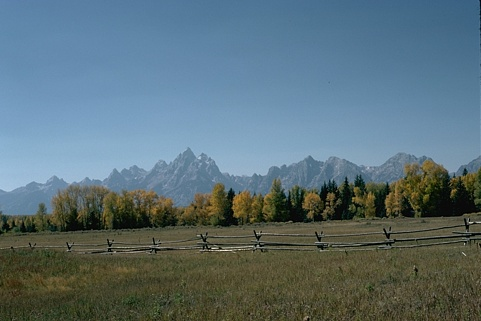}}
	\subcaptionbox{$[\X]_1$}{\includegraphics[width=.26\linewidth]{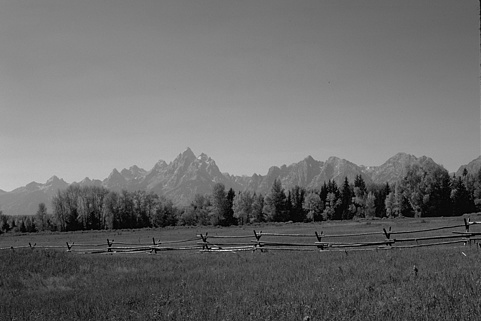}}	
	\subcaptionbox{$[\X]_2$}{\includegraphics[width=.26\linewidth]{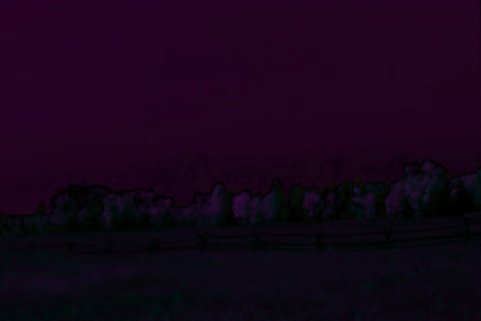}}
	\caption{FCA results of a grayscale video $\X\in\mathbb{R}^{320\times240\times90}$ and a color image $\X\in\mathbb{R}^{321\times481\times3}$.  (a),(d) Original; (b),(e) zero frequency component; (c),(f) nonzero frequency components.}
	\label{pic:FCA}
\end{figure}

We select a grayscale surveillance video sequence from SBI dataset \cite{maddalena2015towards}. It has 90 frames with size $320 \times 240$, denoted as $\X\in\mathbb{R}^{320\times240\times90}$. There  are $I = 46$ frequency bands in Fourier domain, and we analyze the corresponding 46 frequency components in the time domain. Subfigures (a)-(c) of Fig.~\ref{pic:FCA} are an illustration of the frequency component analysis (FCA) results for this grayscale video. We can see that the zero-frequency component contains almost all of the background information (low-rank component), and the moving object (sparse component) lies in the nonzero-frequency component.

In addition, a color image of size $321\times481$ is randomly selected from Berkeley Segmentation Dataset\cite{martin2001database}, denoted as $\X\in\mathbb{R}^{321\times481\times3}$. There are two frequency components including zero frequency component and non-zero frequency component. Subfigures (d)-(f) of Fig.~\ref{pic:FCA} show the FCA results for this color image. As we can see, the zero frequency component contains  the main texture information, and the nonzero frequency component contains the difference information of three channels. Therefore, we can conclude that different bands should be treated differently.

\section{Optimization for FTNN-RTPCA}

\subsection{frequency-filtered RTPCA}

The convex optimization model for RTPCA with the proposed FTNN can be formulated as follows:
\begin{equation}
\min_{\L,\E}\ \norm{\L}_{\FTNN} + \lambda\norm{\E}_1,\ \st\ \X = \L + \E. \label{eq:model}
\end{equation}
where $\X\in\mathbb{R}^{I_1\times I_2\times I_3}$ is the observed tensor, and it can be decomposed into  a low-rank tensor $\L$ and a sparse tensor $\E$. $\lambda$ is a regular parameter that is used to balance the two terms. 

The alternating direction method of multiplier (ADMM) can be applied to solve the convex optimization model (\ref{eq:model}) \cite{ADMM}. The corresponding augmented Lagrangian function is:
\begin{align}
\Gamma(\L,\E,\Y,\mu) = \norm{\L}_{\FTNN} &+ \lambda\norm{\E}_1 + <\Y,\X - \L -\E> \notag\\
&+ \frac{\mu}{2}\norm{\X-\L-\E}_\text{F}^2
\end{align}
where $\Y$ is dual variable and $\mu$ is penalty parameter. Its optimization can be divided into three subproblems, including low-rank component approximation, sparse component minimization and dual variable updatation.

In the $k$-th iteration, the subproblem about the low-rank component $\L$ can be solved by the FTSVT operator as follows:
\begin{align}
	\L_{k+1} &= \argmin_{\L}\ \norm{\L}_{\FTNN} + \frac{\mu_k}{2}\norm{\L - \X + \E_k - \frac{\Y_k}{\mu_k}}_F^2,\notag \\ 
			 &= \FTSVT_{\frac{\bm{\alpha}}{\mu_k}}(\X- \E_k + \frac{\Y}{\mu_k})  \label{eq:updateL}
\end{align}
Because all the tensor norms used in our optimization model can be calculated slices-wisely on the third mode. This t-SVD based RTPCA problem can be divided into $I_3$ matrix robust PCA problems in the Fourier domain.

The sparse component subproblem can be solved by 
\begin{align}
\E_{k+1} &= \argmin_{\E}\ \lambda\norm{\E}_1 + \frac{\mu_k}{2}\norm{\L_{k+1} + \E - \X - \frac{\Y_k}{\mu_k}}_F^2\notag \\
&= \sth_{\frac{\lambda}{\mu_k}}(\X-\L_{k+1} + \frac{\Y}{\mu_k}) \label{eq:updateE}
\end{align}
where  $\sth_{\tau}(\X)$ represents the entry-wise soft thresholding operator. It means that for any entries $x$ in $\X$ we have
\begin{align}
\sth_{\tau}(x) &= \sgn(x)\max(\abs{x}-\tau) 
\end{align}

Finally, we update  the dual variable $\Y_k$ as follows: 
\begin{equation}
	\Y_{k+1} = \Y_k + \mu_k(\X - \L_{k+1} - \E_{k+1})\label{eq:updateY}
\end{equation}

\begin{algorithm}[!t]
	\setstretch{1.3}
	\caption{ADMM for FTNN-RTPCA}\label{algor:FTRPCA}
	\begin{algorithmic}
		\State\textbf{Input:} Tensor $\X\in\mathbb{R}^{I_1\times I_2\times I_3}$, Filtering vector $\bm{\alpha}$, Parameter $\lambda$.
		\State\textbf{Initialize:} $\L=\E=\Y=0$, $\mu$, $\rho$, $\varepsilon$
		\While{not converged}
			\State 1.~Update $\L_{k+1} = \FTSVT_{\frac{\bm{\alpha}}{\mu_k}}(\X - \E_k - \frac{\Y_k}{\mu_k})$
			\State 2.~Update $\E_{k+1} = \sth_{\frac{\lambda}{\mu_k}}(\X - \L_{k+1} - \frac{\Y_k}{\mu_k})$
			\State 3.~Update $\Y_{k+1} = \Y_k + \mu_k(\L_{k+1} + \E_{k+1} - \X)$
			\State 4.~Update $\mu_{k+1} = \rho\mu_k$
			\State 5.~Check condition : $\norm{\L_{k+1} - \L_k}_\text{F} / \norm{\L_k}_\text{F}\le\varepsilon$
		\EndWhile
		\State\textbf{Output:} $\L,~\E$
	\end{algorithmic}
\end{algorithm}

After updating these three terms, the convergence condition for the low-rank component $\L$ needs to be checked. In addition, we set the parameter $\lambda = 1/\sqrt{\max(I_1,I_2)I_3}$ as it is recommended in \cite{lu2019tensor}. Algorithm \ref{algor:FTRPCA} provides the details of the whole procedure for frequency-filtered robust tensor principal component analysis (FTNN-RTPCA) method.

\subsection{Computational complexity}

Most computational costs of Algorithm~\ref{algor:FTRPCA} lie in the update of low-rank components.  Given an observed tensor $\mathcal{X} \in \mathbb{R}^{I_{1}\times I_{2} \times I_{3}}$, the filtering vector is  $\bm{\alpha}= \tran{[\alpha_1,\alpha_2,\ldots,\alpha_I]}$ where $I=\left\lceil\frac{I_3+1}{2}\right\rceil$. When  the number of elements $0<\alpha_k<+\infty$ is $ P $, we need perform $I_1I_2$ FFTs and $ P $ SVDs at each iteration. When updating the low-rank component $\L$, the computational cost in each iteration is $O\left(I_1 I_2 I_3\log I_3 + {P} I_\text{max}I_\text{min}^2\right)$, where $I_\text{max} = \max(I_1,I_2), I_\text{min} = \min(I_1,I_2)$ and $ 0\le P \le I = \left\lceil\frac{I_3+1}{2}\right\rceil$. When updating the sparse component $\E$, we need to compute the soft-thresholding operation in (\ref{eq:updateE}) and the cost is $O\left(I_1 I_2 I_3\right)$. Above all, the computational complexity of FTNN-RTPCA is $O\left(I_1 I_2 I_3(\log I_3+1) + {P} I_\text{max}I_\text{min}^2\right)$. Compared to RTPCA \cite{lu2019tensor} whose costs in each iteraion is $O\left(I_1 I_2 I_3(\log I_3+1) + I I_\text{max}I_\text{min}^2\right)$, the computational complexity for the proposed FTNN-RTPCA method is less than or equal to the classical RTPCA method \cite{lu2019tensor}. 

In Section~\ref{sec:experiments}, we will show that in some selected applications, $ P $ can be reduced to $ 0 $, which results in an SVD-free version of the proposed FTNN-RTPCA. Therefore, the processing time is greatly reduced.

\section{Strategy and Experiments}\label{sec:experiments}

In this section, the filtering strategy for FTNN will be proposed. To verify the feasibility of the filtering strategy, we designed a simulated experiment to illustrate the superiority of FTNN over TNN. Then we will show how to appply the proposed filtering scheme to real applications including color image denoising and background modeling. All experiments are conducted using MatLab R2014b software on an Intel CPU i5-6300 HQ and 8GB RAM computer.

\subsection{Filtering Strategy}

By introducing a series of filtering coefficients, FTNN improves the flexibility of TNN, but it also brings more parameters to the optimization model. Appropriate setting for the filtering coefficients is important to the performance. The filtering coefficient can be pre-defined according to the task characteristics.

\subsubsection{Prior Knowledge}

Denote the uncorrupted tensor as $\X_{G}\in\mathbb{R}^{I_1\times I_2\times I_3}$, the observed tensor corrupted by sparse noise as $\X_{N}$. The number of frequency bands is $I=\left\lceil\frac{I_3+1}{2}\right\rceil$. The prior knowledge is mainly obtained from two aspects:
\begin{itemize}
	\item The nuclear norm distribution of the uncorrupted tensor on the frequency bands $\norm{\Xbar_{G}^{(j)}}_*,1\le j\le I$.
	\item The deviation of nuclear norms caused by the sparse noise $\norm{\Xbar_{N}^{(j)}}_* -  \norm{\Xbar_{G}^{(j)}}_*$.
\end{itemize}

Unfortunately, when dealing with real world data, the true nuclear norms $\norm{\Xbar_{G}^{(j)}}_*,1\le j\le I$ are usually not directly available. In practice, the prior knowledge is obtained by experimental analysis and statistical analysis on the same kind of uncorrupted data. Such a way of acquiring the prior knowledge will be adopted for our real applications in image denoising and background modeling.

\subsubsection{Selection of Filtering Coefficients}

Since the selection of filtering coefficients on frequency bands is task-driven and the prior knowledge changes with data, generally, there are no analytic expressions between  coefficients and  frequency bands of the tensor.

In principle, the filtering coefficients are selected to make the nuclear norm distribution of recovering results more consistent with the original clean one. It is not hard to specify these coefficients by the prior knowledge of the data in the third dimension. We will show how to apply the proposed filtering scheme to applications including color image denoising and background modeling. In the image denoising experiment, we use prior knowledge to estimate the appropriate coefficients. In the background modeling experiment, we directly determine the optimal coefficients based on prior knowledge.

\begin{figure}[!t]
	\centering
	\includegraphics[width=.6\linewidth]{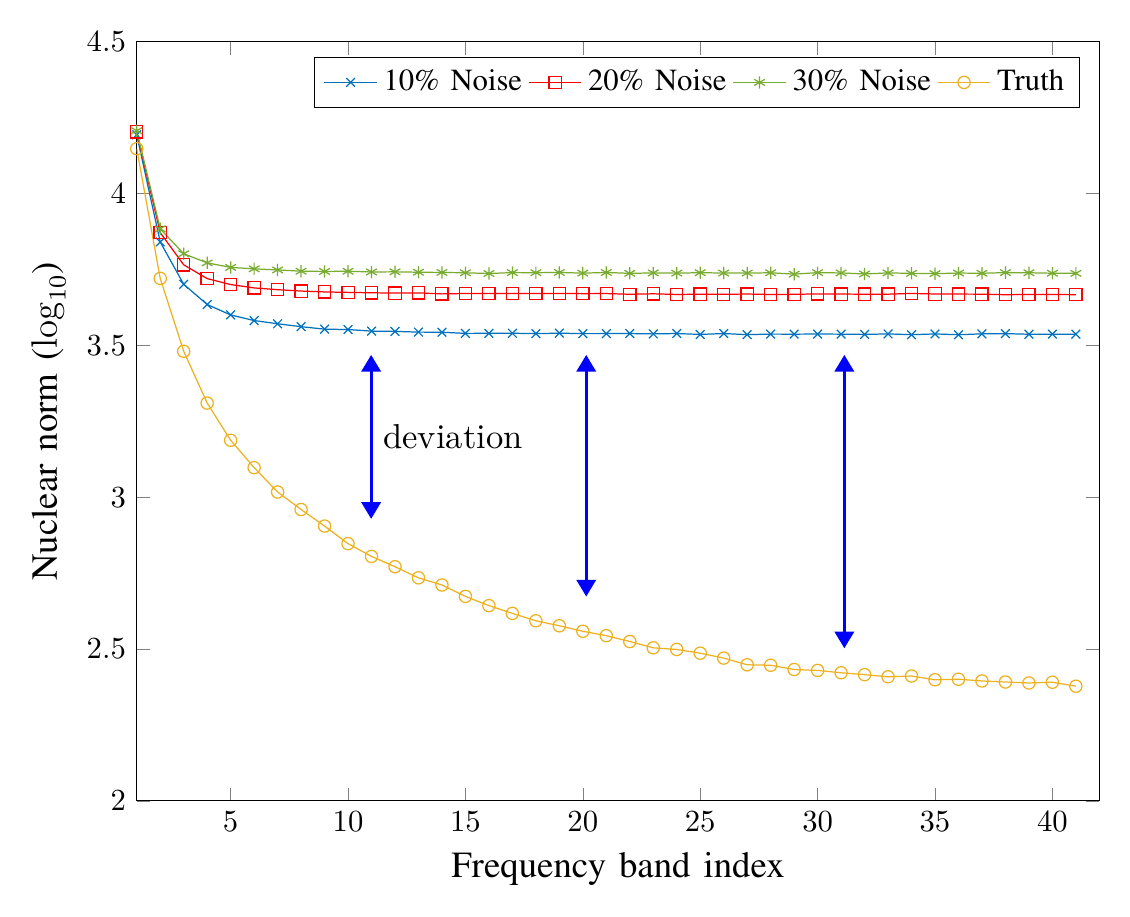}
	\caption{Nuclear norm distribution of frequency band for the synthetic data under different sparse noises.}
	\label{pic:phantom1}
\end{figure}

\subsection{Simulated Experiments}

The synthetic 3D data is generated by utilizing MATLAB function phantom3d \footnote{https://www.mathworks.com/matlabcentral/fileexchange/50974-3d-shepp-logan-phantom}. Each frontal slice of this synthetic data consists of some simple ellipsoids. As the ellipsoids for each frontal slice only change slightly, all frontal slices are very similar and can be regarded as the low-rank component. 

For better illustrating the influence of the sparse noise on different frequency bands, we analyze the nuclear norm for different frequency bands of the synthetic data under different sparse noises. The size of the clean synthetic data is $200\times200\times80$ and there will be 41 frequency bands after applying DFT on the third mode. The corrupted ratio $R_{s}=30\% $ means 30\% pixels of the clean data are  randomly selected to be random values in [0,255]. The order-3 clean tensor is corrupted with $10\%$, $20\%$, $30\%$ sparse noise respectively. Then we calculate the nuclear norms of 41 frequency bands of these corrupted data in the Fourier domain. Fig.~\ref{pic:phantom1} shows the results.

In Fig.~\ref{pic:phantom1}, the blue two-way arrow shows the deviation by  sparse noise on the true synthetic data. The length of the arrow illustrates  influence of sparse noise on the different frequency bands. It shows that sparse noise has little effect on the zero frequency component and the higher bands are more likely to be affected by the noise as the nuclear norms change more. The uncorrupted tensor enjoys a rapid decay of the nuclear norms on the third mode. It means that information is more intensive in the low frequency bands. Therefore, we consider assigning greater weights to higher frequency bands to ensure better recovery performance.

\subsubsection{Comparison}

In order to verify the performance of the FTNN-RTPCA algorithm, we conduct the synthetic 3D data denoising experiment in this section. The size of the original clean synthetic 3D data is ${200\times200\times21}$.  $\X$ represents the corrupted data with $R_{s}=30\% $ and there are 11 frequency bands in all. We compare the proposed FTNN-RTPCA with TNN-RTPCA\cite{lu2019tensor}. For fairness, the parameter is set to the same values as $\lambda = \sqrt{\max(I_1,I_2)I_3}$. Here we use Peak Signal-to-Noise Ratio (PSNR) and relative square error (RSE) to evaluate the recovery accuracy. When $\hat{\L}\in\mathbb{R}^{I_1\times I_2\times I_3}$  represents the recovered tensor and the original data is denoted by $\L_{0}\in\mathbb{R}^{I_1\times I_2\times I_3}$,  PSNR and RSE are defined as:
\begin{equation}\begin{split}
\text{PSNR} &= 10\log_{10}\left(\frac{I_1I_2I_3\norm{\L_0}_{\infty}^2}{\norm{\hat{\L}-\L_0}_F^2}\right) \\
\text{RSE}  &= \frac{\norm{\hat{\L}-\L_0}_F}{\norm{\L_0}_F}
\end{split}\label{eq:PSNR}
\end{equation} 

In this experiment, we test how the recovery performance is influenced by the filtering strategy. The filtering vector is obtained by continually adjusting the value of the filtering vector to make it more consistent with the true situation of the nuclear norm in Fig.~\ref{pic:phantom_norm}. When the two curves almost match, the filtering vector $\bm{\alpha}\in\mathbb{R}^{11}$ is set as 
\begin{equation}
\bm{\alpha} = \tran{[0.3,0.5,0.6,0.75,0.9,1,1.05,1.05,1.1,1.1,1.1]}\label{eq:syn},
\end{equation}
where lower frequency bands are assigned smaller weights and higher bands are assigned larger ones. The texture information mainly lies in the lower frequency bands, the sparse noise mainly lies in the higher frequency bands. Such a setting can make a balance between filtering out the sparse noise and retaining the texture information.

\begin{table}[!ht]
	\centering\renewcommand{\arraystretch}{1.2}\caption{Comparison of TNN and FTNN in terms of PSNE and RSE on synthetic 3D data}\label{tab:RSE}
	\begin{tabular}{|@{\hspace{.4cm}}c|c|c@{\hspace{.4cm}}|}
		\hline 
		&   FTNN     &    TNN     \\ \hline 
		PSNR &  31.1040   &  27.7939	\\ \hline 
		RSE  &  0.1149   &  0.1696	\\ \hline 
	\end{tabular} 
\end{table}


\begin{figure}[!t]
	\centering
	\includegraphics[width=.8\columnwidth]{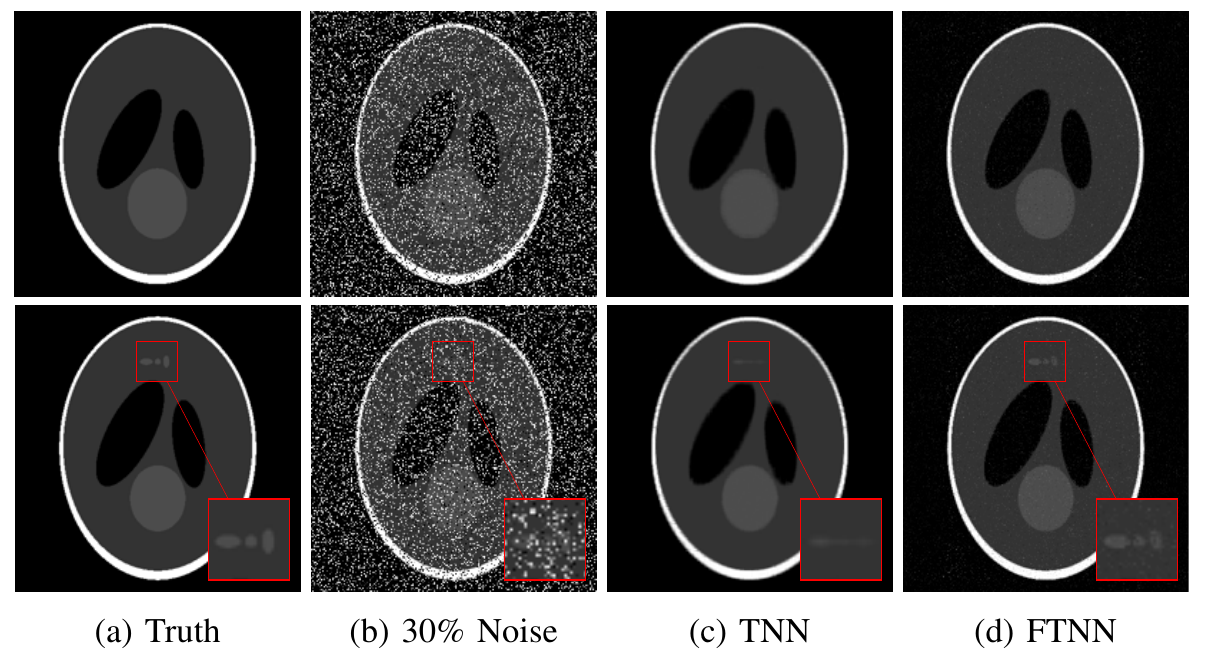}
	\caption{Recovery results of TNN and FTNN method on synthetic 3D data. The top and bottom are the results of frame~1 and~11, respectively.}
	\label{pic:phantom}
\end{figure}

Fig.~\ref{pic:phantom} shows the recovery results of our proposed FTNN and the classical TNN method on synthetic 3D data when the corrupted ratio is $30\% $. As we can see, TNN can remove the sparse noise well but regards the changes between different frontal slices as sparse components and removes them together, which would lead to large information loss of the original clean data. However, FTNN can keep these texture information from being removed as sparse components.  

In addition, the comparison of PSNR and RSE between TNN and FTNN are shown in Table~\ref{tab:RSE}. The proposed FTNN method has the better recovery accuracy performance compared with TNN. We can conclude that based on the prior knowledge of the data in the frequency domain, FTNN can adjust the filtering vector to remove the sparse noise while retaining the change information along the third dimension. Therefore, the proposed FTNN is more flexible and efficient for denoising tasks compared with TNN.
 
\begin{figure}[!t]
	\centering
	\includegraphics[width=.6\columnwidth]{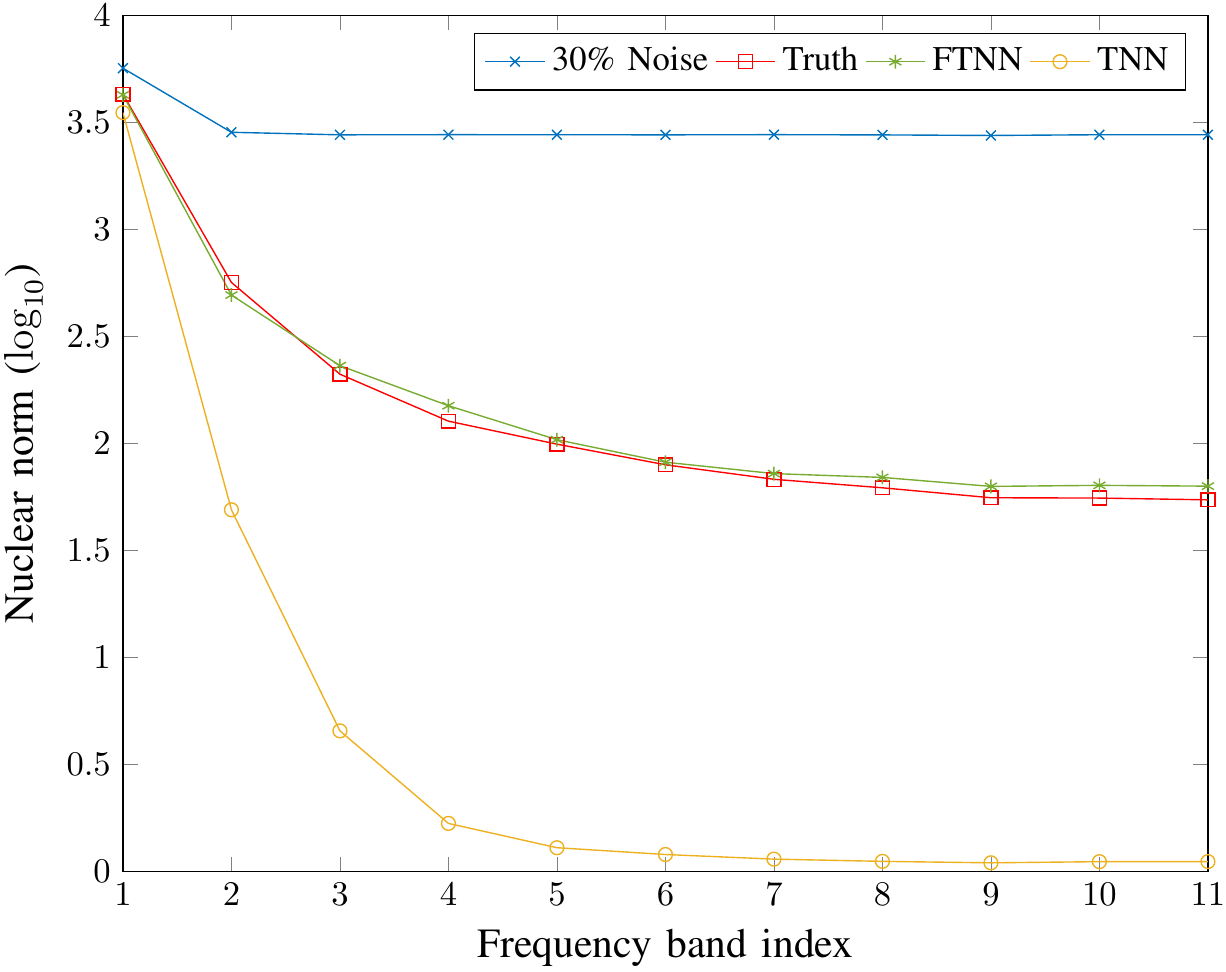}
	\caption{Nuclear norm distribution with respect to frequency band for the recovery results of TNN and FTNN.}\label{pic:phantom_norm}
\end{figure}

In the end, we analyze the nuclear norms of frequency bands for different recovery results. As Fig.~\ref{pic:phantom_norm} shows, the nuclear norm distribution for the proposed FTNN method is more approximate to the true distribution. While the TNN curve is lower than the truth in the high frequency bands, indicating a more information loss in details of the data. The comparison in Fig.~\ref{pic:phantom} and Table~\ref{tab:RSE} shows a better recovery performance of FTNN than TNN. Therefore, we conclude that the more consistent the norm distribution curve is with the truth curve, the better the recovery results will be. These results also verify the effectiveness of our filtering strategy.

When dealing with real-world data, the truth nuclear norm curve is usually unavailable. It is not easy to determine the filtering vector directly which is most consistent with the truth. To find an appropriate filtering vector under this circumstance, we need to acquire the prior knowledge in some other ways. In Section~\ref{sec:ID} and Section~\ref{sec:BM}, we will show how to apply the proposed filtering strategy to practical applications including color image denoising and background modeling.

\subsection{Color Image Denoising}\label{sec:ID}

A color image has three channels which possess strong correlation on the third mode. Since each channel of an image can be approximated by a low-rank matrix \cite{lu2014generalized}, it can be regarded as a low-rank tensor. In real world, images always suffer from sparse noise. In this section, we apply the proposed FTNN method to solve the color image denoising problem. 

\subsubsection{Acquisition of prior knowledge}

The prior knowledge is obtained by statistical analysis on the same kind of uncorrupted data. There are two frequency bands for an color image, denote as $\Xbar^{(1)}$ and $\Xbar^{(2)}$. To explicitly determine the value of the corresponding weights $\alpha_1$ and $\alpha_2$, we randomly select 10 images and add different ratio of sparse noise on them. The ratio of sparse noise range from $10\%$ to $50\%$. Then for each image, the nuclear norm information of two bands is collected as:
\begin{table}[!h]
	\centering
	\setlength{\tabcolsep}{4mm}
	\renewcommand{\arraystretch}{1.5}
	\begin{tabular}{|@{\hspace{0.2cm}}c|c|c@{\hspace{0.2cm}}|}
		\hline 
		Nuclear norms   & First Band & Second Band  \\ 
		\hline 
		Groundtruth     & $\norm{\Xbar^{(1)}_G}_*$ & $\norm{\Xbar^{(2)}_G}_*$ \\ 
		\hline 
		corrupted data  & $\norm{\Xbar^{(1)}_N}_*$ & $\norm{\Xbar^{(2)}_N}_*$ \\ 
		\hline 
	\end{tabular}
\end{table}

\noindent In two-band case, we propose the filtering scheme for $\alpha_1$ and $\alpha_2$ as :
\begin{equation}
\frac{\alpha_1}{\alpha_2} = \frac{\norm{\Xbar^{(1)}_N}_* - \norm{\Xbar^{(1)}_G}_*}{\norm{\Xbar^{(2)}_N}_* - \norm{\Xbar^{(2)}_G}_*}
\end{equation} 
Usually we set $\alpha_2=1$ and then the value of $\alpha_1$ can be specified by above formula. In practice, we find that $\alpha_1$ always ranges from $0.3$ to $0.6$ and the larger the noise ratio is, the larger $\alpha_1$ is. Finally we give the filtering vector for image denoising as :
\begin{equation}
\begin{split}
\text{10\% Noise : } \bm{\alpha} &= \tran{[0.35, 1]} \\
\text{20\% Noise : } \bm{\alpha} &= \tran{[0.45, 1]} 
\end{split}
\end{equation}

\subsubsection{Comparison}
 
Fifty color images are chosen randomly from the Berkeley Segmentation Dataset \cite{martin2001database} for this group of experiments. For every color image, we randomly choose 10\% and 20\% pixels and set their values as random values in the range [0,255]. To verify the performance of our FTNN method, we select some other state-of-the-art methods for comparison including RPCA \cite{candes2011robust}, SNN\cite{liu2012tensor}, TNN\cite{lu2019tensor}. RPCA is a matrix-based method while the others are tensor-based methods. Suppose the size of the corrupted image is $I_1\times I_2\times 3$, we set  $\lambda= 1 / \sqrt {\text{max}(I_{1},I_{2})}$ for RPCA and $\mathbf{\lambda} = \tran{[15,15,1.5]}$ for SNN. As for TNN and the proposed FTNN, the parameter $\lambda$ is set to be $\sqrt{3 \times \max(I_1,I_2)}$.

\begin{figure*}[!ht]
	\centering
	\includegraphics[width=\textwidth]{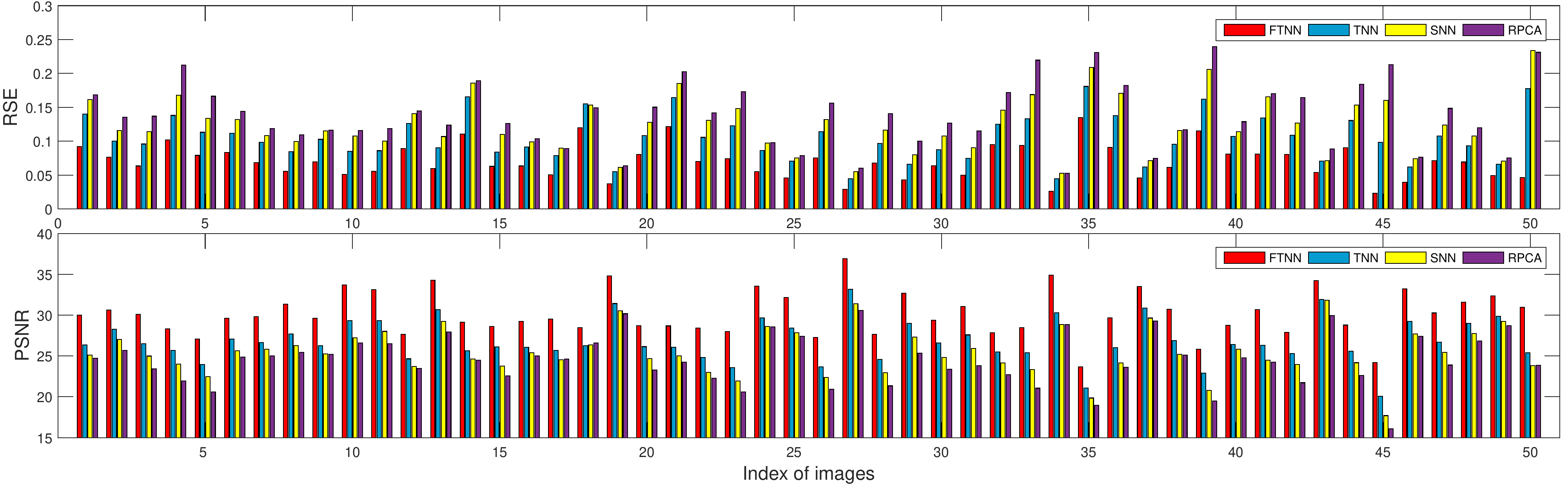}
	\caption{Comparision of PSNR and RSE with 20\% noise}\label{pic:psnr_compare}
\end{figure*}

\begin{table}[!ht]
	\centering\renewcommand{\arraystretch}{1.2}
	\def\h{\phantom{1}}\setlength{\tabcolsep}{2mm}
	\caption{Comparison of recovery performance for varying noises on color images.}\label{tab:psnr_compare}
	\scalebox{.85}{
		\begin{tabular}{|@{\hspace{.1cm}}c|ccccc@{\hspace{.1cm}}|}
			\hline
			Corrputed Ratio 
			&    Indictors 	 &  RPCA     &  SNN    &   TNN    &  FTNN   \\ \hline
			\multirow{2}*{10\%}
			& PSNR   &  25.8731  & 26.8047 & 28.6684 &  \e{35.2650}  \\ 
			& RSE    &  0.12363  & 0.11017 & 0.08834 &  \e{0.04117}\\ \hline
			\multirow{2}*{20\%}
			& PSNR   &  24.5125  & 25.4716 & 26.8288 &  \e{30.1411}  \\ 
			& RSE    &  0.14331  & 0.12783 & 0.10880 &  \e{0.07431}\\   \hline
	\end{tabular}}
\end{table}

We use PSNR and RSE to evaluate the recovery performance of these methods. Table~\ref{tab:psnr_compare} shows the average value of two  metrics of 50 examples with different methods under varying noises. It shows that the proposed FTNN method obtains the highest PSNR values and the lowest RSE value. We can conclude that FTNN achieves the best recovery performance. In addition, we present the comparison results of different methods in terms of PSNR and RSE for 50 groups of experiments when the corrupted ratio is 20\% in Fig.~\ref{pic:psnr_compare}. We can see that our FTNN always achieve the highest recovery accuracy. 

Fig.~\ref{pic:denoise_example} shows the results of five examples, which are animal, house, child, flower and person, respectively.  It can be seen that the recovery images by RPCA, SNN, TNN methods are relatively fuzzy while the proposed FTNN can obtain the best recovery images. In particular, the results of FTNN are very clear in details such as the crosses on the roof and the stamens of the flowers.  

Based on the above comparison, we summarize several conclusions. First, RPCA gets the worst recovery performance because it processes each channel separately and the relevant information about three channels are ignored. Secondly, FTNN and TNN can obtain better results compared with RPCA and SNN. The reason is that the t-SVD framework make more efficient use of the multi-dimensional structure of the data. Finally, the proposed FTNN can achieve the best performance and recover more details because each band is treated differently. 

Since the filtering vector is chosen according to the prior knowledge of the data in the Fourier domain, more details have been recovered in the results by FTNN method. Therefore, FTNN is an efficient and flexible algorithm to perform well by designing an appropriate weighted vector.

\begin{figure*}[!t]
	\centering
	\includegraphics[width=\textwidth]{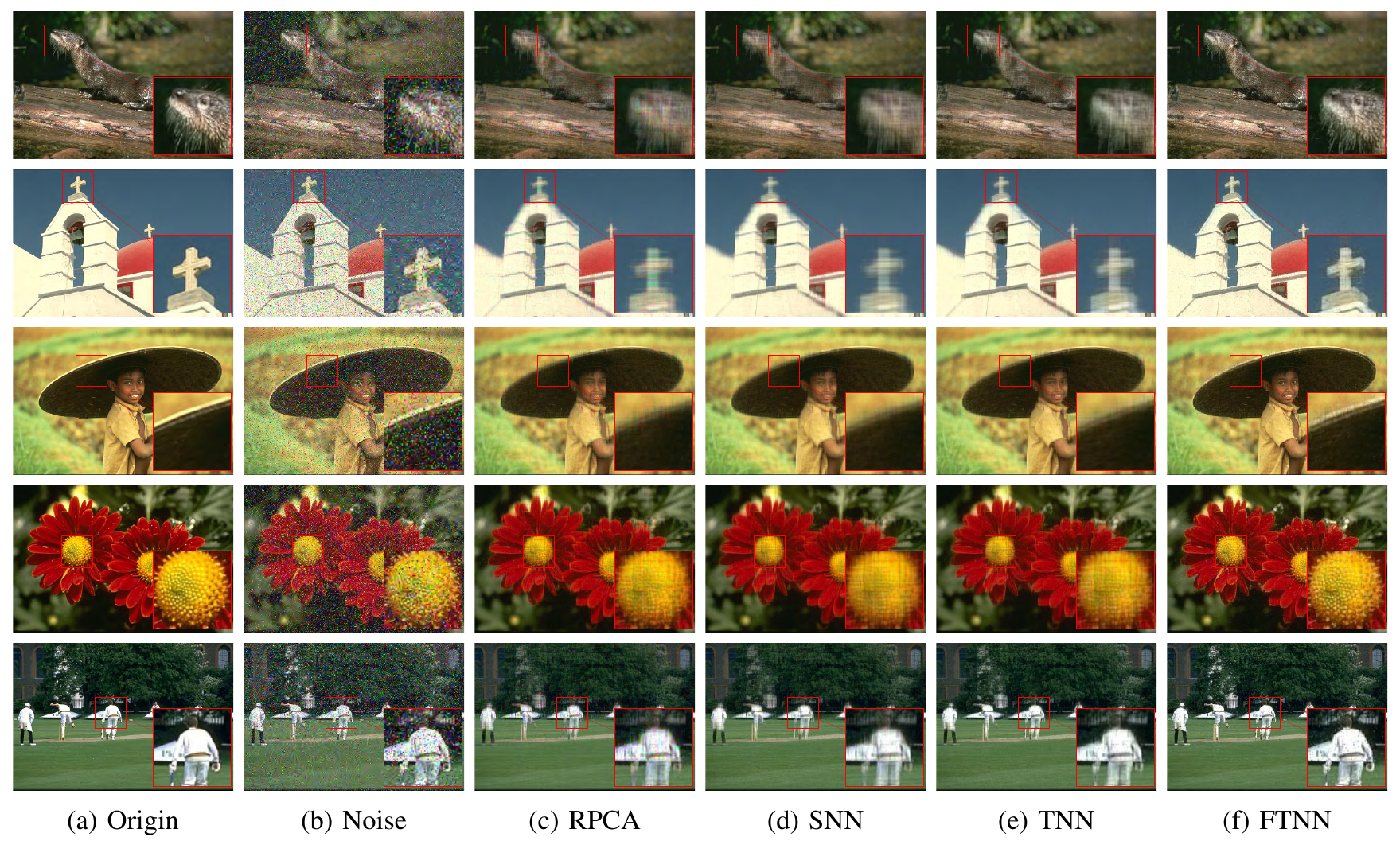}
	\caption{Comparison of some examples on color image denoising with $20\%$ noise. (a) Original images; (b) Noisy images; (c), (d), (e) and (f) are the recovery images by RPCA, SNN, TNN and FTNN respectively.}
	\label{pic:denoise_example}
\end{figure*}

\subsection{Background Modeling}\label{sec:BM}

In this section, we apply the proposed FTNN method to solve the background modeling problem. The task of background modeling is to get a good and clean background also known as the foreground-free image from a video sequence. This is a pre-process step in many visual applications such as surveillance video processing, object direction and segmentation. The RPCA-related methods can be used to solve this problem because of the strong correlations between frames. 

\subsubsection{Acquisition of prior knowledge}

The prior knowledge for background modeling is obtained by  analysis of FCA results in Section~\ref{sec:FCA}. Supposing a video sequence has $I_3$ frames in all, there are $ I=\left\lceil\frac{I_3+1}{2}\right\rceil$ bands in the Fourier domain. As we only want to get the foreground-free image from the corresponding video data, the foreground-free video only contains the zero-frequency component. Therefore, we have:
\begin{equation*}
\norm{\Xbar^{(2)}_G}_* = \norm{\Xbar^{(3)}_G}_* = \cdots = \norm{\Xbar^{( I)}_G}_* = 0
\end{equation*}
Different from the filtering scheme proposed in image denoising, we design a special scheme named as zero-frequency filtering. As the moving objects always lie in the nonzero-frequency components, we set $ \alpha_i = +\infty, i = 2 , 3, \ldots, I$ to filter out all the nonzero-frequency bands. Thus we reserve zero-frequency band and discard other bands. Finally, the filtering vector $\bm{\alpha}$ can be set as follows:
\begin{equation}
\bm{\alpha} = \tran{[0, +\infty, +\infty, \ldots, +\infty ]} \label{eq:FTNN_BM}
\end{equation} 

In this way, no SVDs are needed when calculating FTSVT operator to update the low-rank component. The computational cost will depend on the FFT in each iteration, which should make the proposed FTNN algorithm run much faster than other RPCA-based methods. We will present the comparison of running time later. 

\subsubsection{Comparison}

To verify performance of our FTNN method in this experiment, we select five color video sequences from dataset SBI \cite{maddalena2015towards}. They are CAVIAR1 sequences with 150 frames of size $384\times256\times3$ , HumanBody2 sequences with 93 frames of size $320\times240\times3$, HighwayI sequences with 88 frames of size $320\times240\times3$, HighwayII sequences with 100 frames of size $320\times240\times3$ and IBMTest2 sequences with 90 frames of size $320\times240\times3$. For comparison, we choose four methods. They are matrix-based method RPCA, tensor-based methods TNN-RTPCA, two state-of-the-art methods including DECOLOR\cite{zhou2012moving} and BRTF\cite{zhao2015bayesian}. All the parameters are set as recommended in the paper.

\begin{table*}[!ht]
	\centering
	\setlength{\tabcolsep}{3mm}
	\renewcommand{\arraystretch}{1.3}
	\caption{Comparison of different evaluation metrics on background modeling}\label{tab:BM}
	\begin{tabular}{|@{\hspace{.3cm}}c|ccccccc@{\hspace{.3cm}}|}
		\hline
		Sequence &	Methods  & AGE & pEPs & pCEPs & MSSSIM & PSNR	    & CQM   	   \\ \hline
		\multirow{5}*{CAVIAR1}
		&	FTNN    & \e{3.0717} & 0.0037 	  & 0.0027     & \e{0.9917} & \e{33.6179} & \e{33.8376}	 \\ 
		&	BRTF    & 3.7218 	 & 0.0052     & 0.0033 	   & 0.9793 	& 32.6354     & 32.9582      \\ 
		&	DECOLOR & 3.2065 	 & \e{0.0030} & \e{0.0020} & 0.9896 	& 32.4944     & 32.7749 	 \\
		&	TNN     & 6.2369 	 & 0.0491     & 0.0410 	   & 0.8883 	& 26.0245     & 26.6583 	 \\
		&	RPCA    & 4.9186 	 & 0.0305     & 0.0253     & 0.9330     & 29.7088 	  & 30.1066    	 \\ \hline
		\multirow{5}*{HumanBody2}
		&	FTNN    & \e{3.3397} & \e{0.0038} & \e{0.0003} & \e{0.9975} & \e{34.5978} & \e{35.0202} \\ 
		&	BRTF    & 5.8037 	 & 0.0438     & 0.0262 	   & 0.9725 	& 26.9631     & 27.5809 	\\
		&	DECOLOR & 3.9077 	 & 0.0064     & 0.0006 	   & 0.9931 	& 31.1452     & 31.0911 	\\
		&	TNN     & 5.0021 	 & 0.0208     & 0.0090 	   & 0.9890 	& 30.6921     & 31.0409 	\\
		&	RPCA    & 5.1913 	 & 0.0252     & 0.0127 	   & 0.9866     & 29.8480 	  & 30.4087     \\ \hline
		\multirow{5}*{HighwayI}
		&	FTNN    & \e{1.7890} & \e{0.0027} & \e{0.0004} & \e{0.9901} & \e{38.2507} & \e{39.1607} \\ 
		&	BRTF    & 2.4575 	 & \e{0.0027} & \e{0.0004} & 0.9881 	& 36.9758     & 37.8537 	\\
		&	DECOLOR & 3.8674 	 & 0.0288     & 0.0175 	   & 0.9501 	& 28.3436     & 29.2440 	\\
		&	TNN     & 4.4316 	 & 0.0056     & 0.0006 	   & 0.9401 	& 32.2916     & 33.3336 	\\
		&	RPCA    & 2.7039 	 & 0.0049     & 0.0013     & 0.9826     & 36.3446 	  & 37.1835     \\ \hline
		\multirow{5}*{HighwayII}
		&	FTNN    & 2.6195     & \e{0.0047} & \e{0.0000} & 0.9934     & 34.6171     & 35.4127 	\\ 
		&	BRTF    & \e{2.4876} & 0.0049     & \e{0.0000} & \e{0.9946} & 34.3662     & 35.0447     \\
		&	DECOLOR & 2.5498 	 & 0.0051     & 0.0001 	   & 0.9933 	& \e{34.9028} & \e{35.6343} \\
		&	TNN     & 3.3198 	 & 0.0058     & \e{0.0000} & 0.9881 	& 32.3148     & 33.1945 	\\
		&	RPCA    & 2.4542 	 & 0.0055     & \e{0.0000} & 0.9938     & 33.5809 	  & 34.4022     \\ \hline
		\multirow{5}*{IBMTest2}
		&	FTNN    & \e{2.7945} & \e{0.0006} & \e{0.0000} & \e{0.9957} & \e{36.6157} & \e{36.4704} \\ 
		&	BRTF    & 4.3095 	 & 0.0048     & 0.0048 	   & 0.9945 	& 33.5641     & 33.4791 	\\
		&	DECOLOR & 4.2134 	 & 0.0031     & \e{0.0000} & 0.9941 	& 34.1732     & 34.2890 	\\
		&	TNN     & 4.7225 	 & 0.0248     & 0.0130 	   & 0.9857 	& 30.9800     & 31.3569 	\\
		&	RPCA    & 4.2342 	 & 0.0049     & 0.0006     & 0.9929     & 33.1650 	  & 33.3380     \\ \hline
	\end{tabular}
\end{table*}
\begin{table*}[!ht]
	\centering\setlength{\tabcolsep}{3mm}
	\def\h{\phantom{1}}\renewcommand{\arraystretch}{1.2}
	\caption{Comparison of running time (seconds) on background modeling}\label{tab:BM_time}
	\begin{tabular}{|@{\hspace{.4cm}}c|ccccc@{\hspace{.6cm}}|}
		\hline
		Sequence   &     RPCA   &    TNN    &   DECOLOR   &   BRTF 	 &   FTNN      \\ \hline
		CAVIAR1    &   159.17   &  333.47   &   1266.84   &  210.24   & \e{43.05}   \\ \hline
		HumanBody2 &  \h97.26   &  156.77   &  \h626.30   & \h96.14   & \e{23.26}   \\ \hline
		HighwayI   &  \h89.95   &  146.86   &  \h462.12   & \h88.06   & \e{20.37}   \\ \hline
		HighwayII  &  \h97.76   &  148.54   &  \h347.20   & \h96.39   & \e{20.72}   \\ \hline
		IBMTest2   &   103.63   &  151.55   &   \h281.21   & \h97.59   & \e{18.19}   \\ \hline
	\end{tabular}
\end{table*}

\begin{figure*}[!t]
	\centering
	\includegraphics[width=\textwidth]{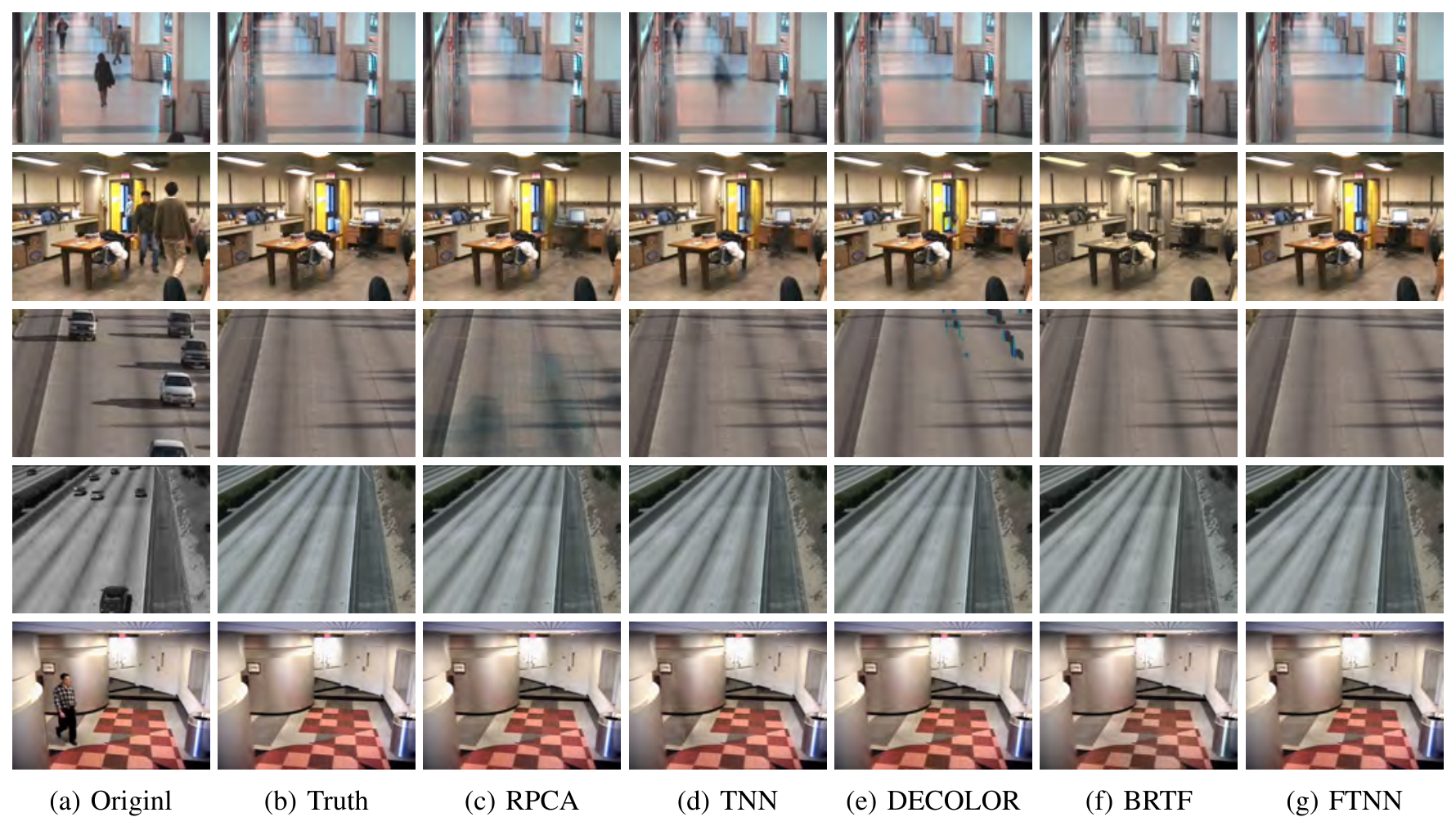}
	\caption{Background modeling on different video sequences. From top to bottom are CAVIAR1, HumanBody2, HighwayI, HighwayII, IBMtest2 video sequences respectively.}\label{pic:BM}
\end{figure*}

Six widely used metrics are used to evaluate the quality of the recovered background from video sequences as follows:

\begin{enumerate}
	\item AGE (Average Gray-level Error), is the average of the gray-level absolute difference between ground-truth and the recovered background. Range in $[0,255]$.
	\item pEPs (Percentage of Error Pixels), is the percentage of error pixels, whose absolute difference between ground-truth and the recovered background is greater than a threshold, with respect to the total number of pixels in the image. Range in $[0,1]$.
	\item pCEPs (Percentage of Clustered Error Pixels), is the percentage of clustered error pixels, whose 4-connected neighbours are also error pixels, with respect to the total number of pixels in the image. Range in $[0,1]$.
	\item MSSSIM (Multi-Scale Structural Similarity Index), is the estimate of the perceived visual distortion. Range in $[-1,1]$. Defined in \cite{MSSSIM,SSIM}.
	\item PSNR: Mentioned before in (\ref{eq:PSNR}).
	\item CQM (Color image Quality Measure), is a metric which is calculated in YUV color space and based on the PSNR computed in the single YUV band. Defined in \cite{CQM}.
\end{enumerate}

For the first three evaluation metrics AGE, pEPs and pCEPs, the lower these values, the better the recovery background. For the last three evaluation metrics MSSSIM, PSNR and CQM, the higher these values, the better the recovered background. 

Table~\ref{tab:BM} shows the evaluation metric results of different methods on five color video sequences. It can be seen that FTNN method can achieve the best precision of at least two metrics for all sequences. Meanwhile, compared with RPCA and TNN, the results by BRTF, DECOLOR and FTNN methods have higher accuracy.

Fig.~\ref{pic:BM} illustrates the recovered background image by five methods on CAVIAR1, HumanBody2, HighwayI, HighwayII and IBMTest2 sequences.

We can see that the background images recovered by TNN always have the foreground ghosting. It indicates that TNN can not remove the sparse foreground component sufficiently. BRTF method removes the foreground very well but it has a distinct color distortion. DECOLOR can get a good background in most cases but performs badly in HighwayI sequence. As for our FTNN method, foreground pixels can be always removed sufficiently in all sequences and  background information can be preserved very well.

Finally, Table~\ref{tab:BM_time} shows the comparison of running time for all methods. It can be seen that the running time of the FTNN method is far less than the other methods, which verify the efficiency of our method. The reason is that we set the filtering vector to  make our FTNN  SVD-free as the computational costs of SVD are very expensive in RPCA-related problems. 

Overall, according to the prior information of different data in the Fourier domain, the proposed method can adaptively choose the filtering vector and achieve better performance in accuracy and running time.

\section{Conclusion}\label{sec:conlusion} 

In this paper, we propose a novel robust tensor principal component analysis (RTPCA) method based on the frequency-filtered tensor nuclear norm (FTNN), which can make better use of the prior information in the frequency domain. We first define the framework of proposed FTNN, which combines the low-rank minimization and Fourier filtering together. Then we afford two viewpoints of FTNN and rigorously deduce the frequency-filtered tensor singular value thresholding (FTSVT) operator. Later we give the frequency component analysis for some visual data. Finally we study the filtering strategy and discuss the selection of the filtering vector. Experiments on synthetic 3D data discuss how to choose a appropriate filtering vector and show the superiority and flexibility of the proposed FTNN. Experiments on color image denoising and background modeling from video sequence shows the proposed method outperforms the existing methods in accuracy and running time.

\bibliographystyle{IEEEtran}

\bibliography{IEEEabrv,reference}

\begin{thebibliography}{10}
\providecommand{\url}[1]{#1}
\csname url@samestyle\endcsname
\providecommand{\newblock}{\relax}
\providecommand{\bibinfo}[2]{#2}
\providecommand{\BIBentrySTDinterwordspacing}{\spaceskip=0pt\relax}
\providecommand{\BIBentryALTinterwordstretchfactor}{4}
\providecommand{\BIBentryALTinterwordspacing}{\spaceskip=\fontdimen2\font plus
\BIBentryALTinterwordstretchfactor\fontdimen3\font minus
  \fontdimen4\font\relax}
\providecommand{\BIBforeignlanguage}[2]{{%
\expandafter\ifx\csname l@#1\endcsname\relax
\typeout{** WARNING: IEEEtran.bst: No hyphenation pattern has been}%
\typeout{** loaded for the language `#1'. Using the pattern for}%
\typeout{** the default language instead.}%
\else
\language=\csname l@#1\endcsname
\fi
#2}}
\providecommand{\BIBdecl}{\relax}
\BIBdecl

\bibitem{PCA}
S.~Wold, K.~Esbensen, and P.~Geladi, ``Principal component analysis,''
  \emph{Chemometrics and intelligent laboratory systems}, vol.~2, no. 1-3, pp.
  37--52, 1987.

\bibitem{fischler1981random}
M.~A. Fischler and R.~C. Bolles, ``Random sample consensus: a paradigm for
  model fitting with applications to image analysis and automated
  cartography,'' \emph{Communications of the ACM}, vol.~24, no.~6, pp.
  381--395, 1981.

\bibitem{huber1981robust}
P.~J. Huber and E.~M. Ronchetti, ``Robust statistics john wiley \& sons,''
  \emph{New York}, vol.~1, no.~1, 1981.

\bibitem{de2003framework}
F.~De~La~Torre and M.~J. Black, ``A framework for robust subspace learning,''
  \emph{International Journal of Computer Vision}, vol.~54, no. 1-3, pp.
  117--142, 2003.

\bibitem{gnanadesikan1972robust}
R.~Gnanadesikan and J.~R. Kettenring, ``Robust estimates, residuals, and
  outlier detection with multiresponse data,'' \emph{Biometrics}, pp. 81--124,
  1972.

\bibitem{ke2005robust}
Q.~Ke and T.~Kanade, ``Robust l1 norm factorization in the presence of outliers
  and missing data by alternative convex programming,'' in \emph{2005 IEEE
  Computer Society Conference on Computer Vision and Pattern Recognition
  (CVPR'05)}, vol.~1.\hskip 1em plus 0.5em minus 0.4em\relax IEEE, 2005, pp.
  739--746.

\bibitem{candes2011robust}
E.~J. Cand{\`e}s, X.~Li, Y.~Ma, and J.~Wright, ``Robust principal component
  analysis?'' \emph{Journal of the ACM (JACM)}, vol.~58, no.~3, pp. 1--37,
  2011.

\bibitem{face}
A.~S. Georghiades, P.~N. Belhumeur, and D.~J. Kriegman, ``From few to many:
  Illumination cone models for face recognition under variable lighting and
  pose,'' \emph{IEEE Transactions on Pattern Analysis and Machine
  Intelligence}, vol.~23, no.~6, pp. 643--660, 2001.

\bibitem{bao2012inductive}
B.~k. Bao, G.~Liu, C.~Xu, and S.~Yan, ``Inductive robust principal component
  analysis,'' \emph{IEEE Transactions on Image Processing}, vol.~21, no.~8, pp.
  3794--3800, 2012.

\bibitem{wang2010efficient}
Z.~Wang and X.~Xie, ``An efficient face recognition algorithm based on robust
  principal component analysis,'' in \emph{Proceedings of the Second
  International Conference on Internet Multimedia Computing and Service}.\hskip
  1em plus 0.5em minus 0.4em\relax ACM, 2010, pp. 99--102.

\bibitem{wright2009robust}
J.~Wright, A.~Ganesh, S.~Rao, Y.~Peng, and Y.~Ma, ``Robust principal component
  analysis: Exact recovery of corrupted low-rank matrices via convex
  optimization,'' in \emph{Advances in neural information processing systems},
  2009, pp. 2080--2088.

\bibitem{li2004statistical}
L.~Li, W.~Huang, I.~Gu, and Q.~Tian, ``Statistical modeling of complex
  backgrounds for foreground object detection,'' \emph{IEEE Transactions on
  Image Processing}, vol.~13, no.~11, pp. 1459--1472, 2004.

\bibitem{cao2015total}
X.~Cao, L.~Yang, and X.~Guo, ``Total variation regularized rpca for irregularly
  moving object detection under dynamic background,'' \emph{IEEE transactions
  on cybernetics}, vol.~46, no.~4, pp. 1014--1027, 2015.

\bibitem{gu2017weighted}
S.~Gu, Q.~Xie, D.~Meng, W.~Zuo, X.~Feng, and L.~Zhang, ``Weighted nuclear norm
  minimization and its applications to low level vision,'' \emph{International
  journal of computer vision}, vol. 121, no.~2, pp. 183--208, 2017.

\bibitem{zhao2014robust}
Q.~Zhao, D.~Meng, Z.~Xu, W.~Zuo, and L.~Zhang, ``Robust principal component
  analysis with complex noise,'' in \emph{International conference on machine
  learning}, 2014, pp. 55--63.

\bibitem{lin2010augmented}
Z.~Lin, M.~Chen, and Y.~Ma, ``The augmented lagrange multiplier method for
  exact recovery of corrupted low-rank matrices,'' \emph{arXiv preprint
  arXiv:1009.5055}, 2010.

\bibitem{Tensor_Decompositions}
T.~G. Kolda and B.~W. Bader, ``Tensor decompositions and applications,''
  \emph{SIAM Review}, vol.~66, no.~4, pp. 294--310, 2005.

\bibitem{Tensor_Decompositions_1}
A.~Cichocki, D.~Mandic, L.~De~Lathauwer, G.~Zhou, Q.~Zhao, C.~Caiafa, and
  H.~Phan, ``Tensor decompositions for signal processing applications: From
  two-way to multiway component analysis,'' \emph{IEEE Signal Processing
  Magazine}, vol.~32, no.~2, pp. 145--163, 2014.

\bibitem{lu2008mpca}
H.~Lu, K.~Plataniotis, and A.~N. Venetsanopoulos, ``Mpca: Multilinear principal
  component analysis of tensor objects,'' \emph{IEEE transactions on Neural
  Networks}, vol.~19, no.~1, pp. 18--39, 2008.

\bibitem{lu2016tensor}
C.~Lu, J.~Feng, Y.~Chen, W.~Liu, Z.~Lin, and S.~Yan, ``Tensor robust principal
  component analysis: Exact recovery of corrupted low-rank tensors via convex
  optimization,'' in \emph{Proceedings of the IEEE conference on computer
  vision and pattern recognition}, 2016, pp. 5249--5257.

\bibitem{cao2016total}
W.~Cao, Y.~Wang, J.~Sun, D.~Meng, C.~Yang, A.~Cichocki, and Z.~Xu, ``Total
  variation regularized tensor rpca for background subtraction from compressive
  measurements,'' \emph{IEEE Transactions on Image Processing}, vol.~25, no.~9,
  pp. 4075--4090, 2016.

\bibitem{lebedev2014speeding}
V.~Lebedev, Y.~Ganin, M.~Rakhuba, I.~Oseledets, and V.~Lempitsky, ``Speeding-up
  convolutional neural networks using fine-tuned cp-decomposition,''
  \emph{arXiv preprint arXiv:1412.6553}, 2014.

\bibitem{luo2017tensor}
Q.~Luo, Z.~Han, X.~Chen, Y.~Wang, D.~Meng, D.~Liang, and Y.~Tang, ``Tensor rpca
  by bayesian cp factorization with complex noise,'' in \emph{Proceedings of
  the IEEE International Conference on Computer Vision}, 2017, pp. 5019--5028.

\bibitem{liu2019low}
Y.~Liu, Z.~Long, H.~Huang, and C.~Zhu, ``Low cp rank and tucker rank tensor
  completion for estimating missing components in image data,'' \emph{IEEE
  Transactions on Circuits and Systems for Video Technology}, 2019.

\bibitem{acar2011scalable}
E.~Acar, D.~M. Dunlavy, T.~G. Kolda, and M.~M{\o}rup, ``Scalable tensor
  factorizations for incomplete data,'' \emph{Chemometrics and Intelligent
  Laboratory Systems}, vol. 106, no.~1, pp. 41--56, 2011.

\bibitem{jain2014provable}
P.~Jain and S.~Oh, ``Provable tensor factorization with missing data,'' in
  \emph{Advances in Neural Information Processing Systems}, 2014, pp.
  1431--1439.

\bibitem{de2008tensor}
V.~De~Silva and L.~Lim, ``Tensor rank and the ill-posedness of the best
  low-rank approximation problem,'' \emph{SIAM Journal on Matrix Analysis and
  Applications}, vol.~30, no.~3, pp. 1084--1127, 2008.

\bibitem{rai2014scalable}
P.~Rai, Y.~Wang, S.~Guo, G.~Chen, D.~Dunson, and L.~Carin, ``Scalable bayesian
  low-rank decomposition of incomplete multiway tensors,'' in
  \emph{International Conference on Machine Learning}, 2014, pp. 1800--1808.

\bibitem{zhao2015bayesian}
Q.~Zhao, G.~Zhou, L.~Zhang, A.~Cichocki, and S.~Amari, ``Bayesian robust tensor
  factorization for incomplete multiway data,'' \emph{IEEE transactions on
  neural networks and learning systems}, vol.~27, no.~4, pp. 736--748, 2015.

\bibitem{delathauwer2000a}
L.~De~Lathauwer, B.~De~Moor, and J.~Vandewalle, ``A multilinear singular value
  decomposition,'' \emph{SIAM Journal on Matrix Analysis and Applications},
  vol.~21, no.~4, pp. 1253--1278, 2000.

\bibitem{CPtuckerrank}
T.~G. Kolda and B.~Bader, ``Tensor decompositions and applications,''
  \emph{SIAM review}, vol.~51, no.~3, pp. 455--500, 2009.

\bibitem{liu2012tensor}
J.~Liu, P.~Musialski, P.~Wonka, and J.~Ye, ``Tensor completion for estimating
  missing values in visual data,'' \emph{IEEE transactions on pattern analysis
  and machine intelligence}, vol.~35, no.~1, pp. 208--220, 2012.

\bibitem{goldfarb2014robust}
D.~Goldfarb and Z.~Qin, ``Robust low-rank tensor recovery: Models and
  algorithms,'' \emph{SIAM Journal on Matrix Analysis and Applications},
  vol.~35, no.~1, pp. 225--253, 2014.

\bibitem{huang2014provable}
B.~Huang, C.~Mu, D.~Goldfarb, and J.~Wright, ``Provable low-rank tensor
  recovery,'' \emph{Optimization-Online}, vol. 4252, no.~2, 2014.

\bibitem{kilmer2011factorization}
M.~E. Kilmer and C.~D. Martin, ``Factorization strategies for third-order
  tensors,'' \emph{Linear Algebra and its Applications}, vol. 435, no.~3, pp.
  641--658, 2011.

\bibitem{braman2010third}
K.~Braman, ``Third-order tensors as linear operators on a space of matrices,''
  \emph{Linear Algebra and Its Applications}, vol. 433, no.~7, pp. 1241--1253,
  2010.

\bibitem{kilmer2013third}
M.~E. Kilmer, K.~Braman, N.~Hao, and R.~C. Hoover, ``Third-order tensors as
  operators on matrices: A theoretical and computational framework with
  applications in imaging,'' \emph{SIAM Journal on Matrix Analysis and
  Applications}, vol.~34, no.~1, pp. 148--172, 2013.

\bibitem{zhang2014novel}
Z.~Zhang, G.~Ely, S.~Aeron, N.~Hao, and M.~Kilmer, ``Novel methods for
  multilinear data completion and de-noising based on tensor-svd,''
  \emph{Computer Science}, vol.~44, no.~9, pp. 3842--3849, 2014.

\bibitem{lu2019tensor}
C.~Lu, J.~Feng, W.~Liu, Z.~Lin, S.~Yan \emph{et~al.}, ``Tensor robust principal
  component analysis with a new tensor nuclear norm,'' \emph{IEEE transactions
  on pattern analysis and machine intelligence}, 2019.

\bibitem{chen2017iterative}
L.~Chen, Y.~Liu, and C.~Zhu, ``Iterative block tensor singular value
  thresholding for extraction of low rank component of image data,'' in
  \emph{2017 IEEE International Conference on Acoustics, Speech and Signal
  Processing (ICASSP)}.\hskip 1em plus 0.5em minus 0.4em\relax IEEE, 2017, pp.
  1862--1866.

\bibitem{feng2020robust}
L.~Feng, Y.~Liu, L.~Chen, X.~Zhang, and C.~Zhu, ``Robust block tensor principal
  component analysis,'' \emph{Signal Processing}, vol. 166, p. 107271, 2020.

\bibitem{liu2018improved}
Y.~Liu, L.~Chen, and C.~Zhu, ``Improved robust tensor principal component
  analysis via low-rank core matrix,'' \emph{IEEE Journal of Selected Topics in
  Signal Processing}, vol.~12, no.~6, pp. 1378--1389, 2018.

\bibitem{kong2018t}
H.~Kong, X.~Xie, and Z.~Lin, ``t-schatten-$ p $ norm for low-rank tensor
  recovery,'' \emph{IEEE Journal of Selected Topics in Signal Processing},
  vol.~12, no.~6, pp. 1405--1419, 2018.

\bibitem{hu2016moving}
W.~Hu, Y.~Yang, W.~Zhang, and Y.~Xie, ``Moving object detection using
  tensor-based low-rank and saliently fused-sparse decomposition,'' \emph{IEEE
  Transactions on Image Processing}, vol.~26, no.~2, pp. 724--737, 2016.

\bibitem{zhou2017outlier}
P.~Zhou and J.~Feng, ``Outlier-robust tensor pca,'' in \emph{Proceedings of the
  IEEE Conference on Computer Vision and Pattern Recognition}, 2017, pp.
  2263--2271.

\bibitem{baburaj2016reweighted}
M.~Baburaj and S.~N. George, ``Reweighted low-rank tensor decomposition based
  on t-svd and its applications in video denoising,'' \emph{arXiv preprint
  arXiv:1611.05963}, 2016.

\bibitem{xu2019laplace}
W.-H. Xu, X.-L. Zhao, T.-Y. Ji, J.-Q. Miao, T.-H. Ma, S.~Wang, and T.-Z. Huang,
  ``Laplace function based nonconvex surrogate for low-rank tensor
  completion,'' \emph{Signal Processing: Image Communication}, vol.~73, pp.
  62--69, 2019.

\bibitem{madathil2018twist}
B.~Madathil and S.~N. George, ``Twist tensor total variation
  regularized-reweighted nuclear norm based tensor completion for video missing
  area recovery,'' \emph{Information Sciences}, vol. 423, pp. 376--397, 2018.

\bibitem{madathil2018dct}
------, ``Dct based weighted adaptive multi-linear data completion and
  denoising,'' \emph{Neurocomputing}, vol. 318, pp. 120--136, 2018.

\bibitem{rojo2004some}
O.~Rojo and H.~Rojo, ``Some results on symmetric circulant matrices and on
  symmetric centrosymmetric matrices,'' \emph{Linear algebra and its
  applications}, vol. 392, pp. 211--233, 2004.

\bibitem{cai2010singular}
J.-F. Cai, E.~J. Cand{\`e}s, and Z.~Shen, ``A singular value thresholding
  algorithm for matrix completion,'' \emph{SIAM Journal on optimization},
  vol.~20, no.~4, pp. 1956--1982, 2010.

\bibitem{maddalena2015towards}
L.~Maddalena and A.~Petrosino, ``Towards benchmarking scene background
  initialization,'' in \emph{International conference on image analysis and
  processing}.\hskip 1em plus 0.5em minus 0.4em\relax Springer, 2015, pp.
  469--476.

\bibitem{martin2001database}
D.~Martin, C.~Fowlkes, D.~Tal, J.~Malik \emph{et~al.}, ``A database of human
  segmented natural images and its application to evaluating segmentation
  algorithms and measuring ecological statistics.''\hskip 1em plus 0.5em minus
  0.4em\relax Iccv Vancouver:, 2001.

\bibitem{ADMM}
S.~Boyd, N.~Parikh, E.~Chu, B.~Peleato, J.~Eckstein \emph{et~al.},
  ``Distributed optimization and statistical learning via the alternating
  direction method of multipliers,'' \emph{Foundations and
  Trends{\textregistered} in Machine learning}, vol.~3, no.~1, pp. 1--122,
  2011.

\bibitem{lu2014generalized}
C.~Lu, J.~Tang, S.~Yan, and Z.~Lin, ``Generalized nonconvex nonsmooth low-rank
  minimization,'' in \emph{Proceedings of the IEEE conference on computer
  vision and pattern recognition}, 2014, pp. 4130--4137.

\bibitem{zhou2012moving}
X.~Zhou, C.~Yang, and W.~Yu, ``Moving object detection by detecting contiguous
  outliers in the low-rank representation,'' \emph{IEEE transactions on pattern
  analysis and machine intelligence}, vol.~35, no.~3, pp. 597--610, 2012.

\bibitem{MSSSIM}
Z.~Wang, E.~P. Simoncelli, and A.~C. Bovik, ``Multiscale structural similarity
  for image quality assessment,'' in \emph{The Thrity-Seventh Asilomar
  Conference on Signals, Systems \& Computers, 2003}, vol.~2.\hskip 1em plus
  0.5em minus 0.4em\relax Ieee, 2003, pp. 1398--1402.

\bibitem{SSIM}
Z.~Wang, A.~Bovik, H.~Sheikh, E.~Simoncelli \emph{et~al.}, ``Image quality
  assessment: from error visibility to structural similarity,'' \emph{IEEE
  transactions on image processing}, vol.~13, no.~4, pp. 600--612, 2004.

\bibitem{CQM}
Y.~Yalman and I.~ERT{\"U}RK, ``A new color image quality measure based on yuv
  transformation and psnr for human vision system,'' \emph{Turkish Journal of
  Electrical Engineering \& Computer Sciences}, vol.~21, no.~2, pp. 603--612,
  2013.

\end{thebibliography}

\end{document}